\documentclass{article}

\usepackage[preprint]{neurips_2026}

\usepackage[utf8]{inputenc} 
\usepackage[T1]{fontenc}    
\usepackage{hyperref}       
\usepackage{url}            
\usepackage{booktabs}       
\usepackage{amsfonts}       
\usepackage{nicefrac}       
\usepackage{microtype}      
\usepackage{xcolor}         

\usepackage{url}            
\usepackage{booktabs}       
\usepackage{amsfonts}       
\usepackage{nicefrac}       
\usepackage{microtype}      

\usepackage{algorithm}

\usepackage{amssymb}		
\usepackage{graphicx}		
\usepackage{amsmath}
\usepackage{amsthm}
\usepackage{algorithmic}
\usepackage{comment}
\usepackage{multirow}
\usepackage{enumitem}
\usepackage{subcaption}
\usepackage{tikz}
\usepackage[dvipsnames,table,xcdraw]{xcolor}

\usepackage{amsmath}  
\usepackage{xspace}   

\usepackage{pifont}
\usepackage{float}

\usepackage{amsmath,amssymb,mathtools}

\usepackage[nameinlink]{cleveref}  

\newtheorem{theorem}{Theorem}

\newtheorem{lemma}{Lemma}

\newtheorem{corollary}{Corollary}

\newtheorem{proposition}{Proposition}

\crefname{section}{Section}{Sections}
\Crefname{section}{Section}{Sections}

\crefname{subsection}{Section}{Sections}
\Crefname{subsection}{Section}{Sections}

\title{Data Attribution in Adaptive Learning}

\author{%
  Amit Kiran Rege \\
  Department of Computer Science\\
  University of Colorado Boulder\\
  Boulder, Colorado 80309 \\
  \texttt{amit.rege@colorado.edu} \\
}

\begin{document}

\maketitle

\begin{abstract}
  Machine learning models increasingly generate their own training data---online bandits, reinforcement learning, and post-training pipelines for language models are leading examples. In these adaptive settings, a single training observation both updates the learner and shifts the distribution of future data the learner will collect. Standard attribution methods, designed for static datasets, ignore this feedback. We formalize occurrence-level attribution for finite-horizon adaptive learning via a conditional interventional target, prove that replay-side information cannot recover it in general, and identify a structural class in which the target is identified from logged data.
\end{abstract}

\section{Introduction}

Modern machine learning systems are trained on large datasets, and the composition of those datasets matters. Questions around data privacy, model debugging, and data curation all reduce to a common underlying question: how much does a specific training example actually shape what a model learns? Data attribution makes this precise. Given a trained model, it asks how the model would have changed if a particular training point had been removed, reweighted, or modified. In the standard supervised setting this is a well-studied problem. Influence functions, TracIn, Data Shapley, and related methods estimate how the final predictor changes when a sample is perturbed in a fixed training set, and a large literature has developed around the approximation strategies and theoretical foundations behind these estimates. The shared structure underlying all of them is that after one perturbs a training point, the future training data do not themselves change.

Adaptive learning complicates the picture in a way these methods are not designed to handle. In online bandits, reinforcement learning, and post-training procedures where later data depend on the current policy, training data is not collected from a fixed source. The learner interacts with an environment, and what it observes next depends on what it has already learned. A single training observation therefore does two things at once: it updates the learner, and through the updated learner it reshapes which data will be collected going forward. Perturbing that observation can change the final model along two distinct channels---through the learning dynamics on the data that follows, and through the future data distribution itself.

A two-step bandit illustrates the point. Suppose the learner observes a reward at round \(1\), updates its policy, and then acts again at round \(2\). If one downweights the first interaction, there are two different counterfactuals one might study. One can hold the realized second-round interaction fixed and ask how the final policy changes along that fixed continuation. Or one can hold fixed only what happened through round \(1\), allow the perturbed learner to act again, and compare the outcome. In a fixed-dataset problem these two constructions coincide. In adaptive learning they do not.

Attribution in adaptive learning should therefore attach to a realized occurrence rather than to an abstract sample identity. The same nominal content can appear at different times and under different learner states, and in an adaptive system those are causally distinct objects. Once one takes this view, two choices must be made before any approximation method enters the picture: whether to study replay on a fixed realized continuation or recollection under the perturbed learner, and how much of the realized history to condition on when defining the target.

The target studied here conditions on the realized prefix through the perturbed round and recollects the future under the perturbed learner. This identifies which occurrence is being attributed and leaves the full post-occurrence future free to adapt. The perturbation is a one-parameter reweighting: negative values downweight the learning effect of the occurrence, and the endpoint suppresses its direct update contribution while leaving it part of the observed prefix.

The paper develops several consequences of this choice. The structural decomposition shows that the interventional target equals conditional expected replay plus a correction term driven by the perturbation-induced shift in the future data law. A stronger negative result shows this gap cannot be closed by enriching the replay information: even knowing the baseline future law and the full family of fixed-log replay responses for every continuation, replay-side data does not identify the interventional target in general. On the positive side, in a contextual-bandit model where the learner affects future data only through a known action law, the perturbed future law has an exact change-of-measure form and the target is identified under overlap; if unknown learner-state dependence enters context or reward generation, identification fails. The paper also gives a bandit example in which replay and intervention have opposite signs, and a depth-\(L\) truncation framework that gives local attribution windows a precise interpretation and quantitative error bounds.

The connection to the causal literature on sequential interventions is that conditioning on an evolving history while modifying later treatment mechanisms is exactly the structure needed here. The contribution is to bring that framework to occurrence-level attribution in adaptive learning and to map out the boundary between identifiable and non-identifiable cases.

The rest of the paper is organized as follows. Section~\ref{sec:setup} defines the adaptive learning model, the replay and interventional targets, the replay oracle, and the action-only class that later marks the positive side of the identification frontier. Section~\ref{sec:related-work} positions the paper relative to static data attribution, trajectory-specific attribution on fixed training streams, recent work on sequential attribution in RL and post-training, and the broader causal literature on sequential interventions. Later sections develop the structural decomposition, the replay-oracle insufficiency theorem, the identification frontier, exact model-based computation, and the bandit examples.

\section{Related work}
\label{sec:related-work}

The paper sits at the intersection of several literatures that ask different questions under the broad heading of data influence. The most direct ancestor is the fixed-dataset attribution literature. Influence functions, TracIn, and valuation-based methods such as Data Shapley ask how a model would have changed if part of an exogenous training set had been removed or reweighted~\cite{koh2017understanding,pruthi2020estimating,ghorbani2019data}. Those works differ in approximation strategy and scale, but they all study learning from a data source that is fixed independently of the learner.

Recent work has sharpened the estimand question within that fixed-data setting. Bae et al.\ show that practical influence-function calculations often approximate a response functional different from literal leave-one-out retraining, while Schioppa et al.\ analyze why the usual leave-one-out reading can fail in nonlinear training~\cite{bae2022if,schioppa2023theoretical}. That line of work is close in spirit to the present paper. The common lesson is that before one asks whether an approximation is accurate, one has to say what object is being approximated.

Another nearby line keeps the data source exogenous but ties attribution to the realized training trajectory rather than to a permutation-invariant sample identity. Hara et al.\ introduced an SGD-based notion of occurrence-specific influence along a training run, Wang et al.\ formalized trajectory-specific leave-one-out influence, and In-Run Data Shapley studies run-specific attribution for a target model produced by a single training run~\cite{hara2019sgd,wang2025capturing,wang2024one-run-shapley}. These papers are especially relevant because they take realized occurrences seriously. The key difference is that perturbing one occurrence still does not change the law of future data. The training trajectory changes, but the data source remains exogenous. In the present paper, replay is the adaptive analogue of that fixed-log viewpoint, while conditional intervention is the target that appears once the learner can reshape the future data stream.

Recent work on sequential learning and RL moves closer to the setting of this paper but studies different targets. Gottesman et al.\ analyze influential transitions for off-policy evaluation, which is a sensitivity question about an evaluation estimator rather than about the learned system after perturbing part of its own training history~\cite{gottesman2020interpretable}. Hiraoka et al.\ study the influence of experiences in replay-based RL agents~\cite{hiraoka2024which}. Hu et al.\ study online RL directly and emphasize that each experience affects both the policy update and future data collection, but their framework is intentionally local, tying attribution to a recent PPO buffer and a nearby checkpoint~\cite{hu2025snapshot}. The present paper instead studies a global occurrence-level target that conditions on the realized prefix through the perturbed round and lets the whole post-occurrence future be recollected under the perturbed learner.

There is also a broader causal literature on sequential interventions. Longitudinal causal inference has long studied interventions that change later treatment mechanisms while conditioning on evolving histories, together with identification by g-computation, weighting, and doubly robust methods. Modern work on longitudinal modified treatment policies is a particularly relevant point of contact because it treats stochastic, history-dependent interventions without requiring point-mass treatment assignments~\cite{diaz2021lmtp}. The present paper uses that framework to formulate occurrence-level attribution in adaptive learning and then asks which replay-side objects fail to identify the resulting target and which structural restrictions make it identifiable from logged data.

Finally, the paper is close in spirit to the counterfactual-learning-systems perspective of Bottou et al., who stressed that replaying logged data is not the same as allowing the system and its future observations to evolve under an intervention~\cite{bottou2013counterfactual}. The difference is one of focus. Bottou et al.\ study counterfactual reasoning for learning systems in general. The present paper specializes that viewpoint to training-data attribution, treats attribution as an occurrence-level question, and develops the structural decomposition, replay-oracle insufficiency result, and identification frontier that arise from that specialization.

\subsection{A three-axis taxonomy of attribution targets}
\label{sec:taxonomy}

Papers on data attribution in adaptive learning often study different objects while using similar language, and the disagreement is usually not about estimator quality but about what counterfactual is being studied. An attribution question has three independent components: the unit of attribution (what is being perturbed---a dataset item, a realized occurrence, a replay-buffer experience, a rollout episode), the future semantics (what happens to later data after the perturbation---a frozen log, a truncated recollection, full recollection under the perturbed learner, an off-policy surrogate), and the target functional (what downstream quantity is being evaluated---the final learned system, a checkpoint, an evaluation estimator, the final training reward). None of these choices determines the others, and a method is not fully specified until all three are fixed. Table~\ref{tab:adaptive-attribution-taxonomy} organizes the nearby literature along these axes.

\begin{table*}[t]
\centering
\footnotesize
\begin{tabular}{p{0.17\linewidth} p{0.17\linewidth} p{0.18\linewidth} p{0.18\linewidth} p{0.23\linewidth}}
\toprule
\textbf{Work family} & \textbf{Unit of attribution} & \textbf{Future semantics} & \textbf{Target functional} & \textbf{Relation to the present paper} \\
\midrule

Trajectory-specific attribution \cite{hara2019sgd,wang2025capturing}
& realized training occurrence
& frozen / exogenous future
& final trained model
& replay-side special case of the exogenous limit \\

Replay-buffer influence \cite{hiraoka2024which}
& stored replay-buffer experience
& fixed replay buffer / replay-side
& final agent performance
& same broad problem family; literal subsumption needs an extension in which the learner state includes buffer contents and replay-sampling events \\

Local online RL attribution \cite{hu2025snapshot,liu2026ippo}
& recent rollout-buffer record or rollout episode
& truncated local window
& checkpoint action / checkpoint return / nearby post-training behavior
& naturally viewed as proxying checkpoint-local or short-horizon truncated targets \\

Replay-based non-local on-policy attribution \cite{hu2026nonlocal}
& realized rollout record
& replay over fixed future buffers
& non-local replay-LOO response
& non-local in time, but still replay-side rather than full recollection \\

OPE influence \cite{gottesman2020interpretable}
& transition in an offline batch
& frozen batch
& evaluation estimator
& different target class: estimator sensitivity rather than learned-system attribution \\

RLVR off-policy influence \cite{zhu2025cropi}
& prompt / trajectory / training sample
& off-policy surrogate
& RLVR training objective
& best viewed as an estimation surrogate near the action-only frontier, not as the same estimand \\

RFT sample influence \cite{tan2025rftinf}
& dataset sample or repeated sample presentation
& optimization-path retracing / proxy
& final training reward
& related global influence target at a different attribution unit \\

World-model data valuation \cite{zhang2026actionshapley}
& offline training sample
& exogenous offline data
& world-model performance
& outside occurrence-level adaptive attribution; relevant mainly as a contrast case \\

\textbf{Present paper}
& \textbf{realized occurrence}
& \textbf{full recollection under the perturbed learner}
& \textbf{final learned system}
& \textbf{canonical global target for occurrence-level attribution in adaptive learning} \\

\bottomrule
\end{tabular}
\caption{A three-axis taxonomy of attribution targets in and around adaptive learning. The present paper occupies one specific corner of this space. Nearby empirical methods often live in neighboring corners that should be understood as replay-side, truncated, or off-policy surrogates rather than as direct estimands of the global interventional target.}
\label{tab:adaptive-attribution-taxonomy}
\end{table*}

The present paper fixes one particular corner of this space: attribution to a \emph{realized occurrence}, full \emph{recollection} of the post-occurrence future under the perturbed learner, and evaluation of the \emph{final learned system}. The taxonomy's value is not that every nearby method reduces to this corner. Its value is that it prevents category mistakes: once the three axes are made explicit, one can ask whether a given method targets the same global object, a principled surrogate, or a different influence question altogether.

\section{Problem Formulation}
\label{sec:setup}

We introduce the setting in full generality and then specialize to a specific example. The finite-horizon assumption lets us present the main ideas without measure-theoretic machinery.

\subsection{Finite-horizon adaptive learning}

Fix a horizon \(T \ge 1\). For each round \(t \in \{1,\dots,T\}\), let \(\mathcal Z_t\) be a finite interaction space. We write
\[
\mathcal H_t := \mathcal Z_1 \times \cdots \times \mathcal Z_t,
\qquad
\mathcal H_0 := \{\emptyset\},
\]
and denote a realized prefix by \(z_{1:t} \in \mathcal H_t\).

The learner has internal state \(\theta_t \in \Theta\), where \(\Theta\) is an arbitrary state space. The environment is indexed by \(\nu\). For each round \(t\), learner state \(\theta \in \Theta\), and history prefix \(z_{1:t-1} \in \mathcal H_{t-1}\), let
\[
K_{\nu,t}^{\theta}(\cdot \mid z_{1:t-1})
\]
be a probability distribution on \(\mathcal Z_t\). This is the law of the next interaction under the current learner state and the current history.

The learner updates through maps
\[
U_t : \Theta \times \mathcal Z_t \times [0,1+\rho] \to \Theta,
\qquad t=1,\dots,T,
\]
for some fixed \(\rho>0\). The third argument is a nonnegative \emph{weight} that scales the direct learning effect of the \(t\)-th interaction. Under a standard run the weight is 1; setting it to 0 asks what would have happened if the learner observed the interaction but did not update from it. The range $[0, 1+\rho]$ allows both downweighting and upweighting.

Given a weight vector \(w=(w_1,\dots,w_T)\in[0,1+\rho]^T\), we write $\theta_t^{(w)}$ for the learner state at time $t$ under weights $w$. The system starts at $\theta_1^{(w)}(\emptyset) := \theta_1$. Online, the data at step $t$ is $Z_t \sim K_{\nu,t}^{\theta_t^{(w)}(Z_{1:t-1})}(\cdot \mid Z_{1:t-1})$, and the state updates via $\theta_{t+1}^{(w)} := U_t(\theta_t^{(w)}, Z_t, w_t)$.

Given a fixed realized log $z_{1:T}$, the \textbf{replay operator} pushes a perturbed learner through that same sequence without allowing the data to change:
\begin{equation}
\theta_1^{(w)}(z_{1:T}) := \theta_1,
\qquad
\theta_{t+1}^{(w)}(z_{1:T}) := U_t(\theta_t^{(w)}(z_{1:T}), z_t, w_t).
\label{eq:replay-dynamics}
\end{equation}

We fix an evaluation functional $F_\nu$ mapping the final learner state to a real number, for instance the expected reward of the terminal policy. For now $F_\nu$ is treated as a known abstract function; the question of identifying it from logged data is taken up in Section~\ref{sec:action-only-frontier}.

For each weight vector \(w\), the recursion in $Z_t \sim K_{\nu,t}^{\theta_t^{(w)}}(\cdot \mid Z_{1:t-1})$ induces a law on full histories:
\begin{equation}
\mathbb P_\nu^{(w)}(z_{1:T})
=
\prod_{t=1}^{T}
K_{\nu,t}^{\theta_t^{(w)}(z_{1:t-1})}(z_t \mid z_{1:t-1}).
\label{eq:path-law}
\end{equation}
The baseline run corresponds to the all-ones vector $\mathbf 1 := (1,\dots,1)$, and we write $\mathbb P_\nu := \mathbb P_\nu^{(\mathbf 1)}$.

We also abbreviate $\theta_s(z_{1:T}) := \theta_s^{(\mathbf 1)}(z_{1:T})$ for the baseline replayed state along a fixed realized history.

To perturb one realized occurrence at round \(t\), we use the one-coordinate family
\begin{equation}
w_s^{(t,\epsilon)} := 1 + \epsilon \mathbf 1\{s=t\},
\qquad
\epsilon \in [-1,\rho].
\label{eq:one-coordinate-family}
\end{equation}
Positive \(\epsilon\) upweights the direct learning effect of the realized occurrence, negative \(\epsilon\) downweights it, and \(\epsilon=-1\) suppresses its direct update contribution.

\subsection{Example: Instantiating the Model}
To make the above formulation concrete, consider a standard two-armed Bernoulli bandit trained via online gradient ascent.

The learner's state is its current policy, so $\theta_t \in (0,1)$ is the probability $p_t$ of pulling Arm 1. Each interaction records both the action and the reward: $Z_t = (A_t, R_t) \in \{0,1\} \times \{0,1\}$. The environment $\nu$ is defined by the true reward probabilities $\mu_0$ and $\mu_1$ for each arm; the kernel $K_{\nu,t}^{\theta_t}$ first samples $A_t \sim \text{Bernoulli}(\theta_t)$ and then $R_t \sim \text{Bernoulli}(\mu_{A_t})$.

The learner updates via entropic mirror descent with an importance-weighted reward estimate. Given the current policy, the observed pair, and a learning weight $w_t$, the next policy is
\begin{equation}
\text{logit}(\theta_{t+1}) = \text{logit}(\theta_t) + \eta w_t R_t \left( \frac{\mathbf{1}\{A_t=1\}}{\theta_t} - \frac{\mathbf{1}\{A_t=0\}}{1-\theta_t} \right),
\label{eq:bandit-logit}
\end{equation}
with $\theta_{t+1} = \sigma(\text{logit}(\theta_{t+1}))$. In a standard training run $w = (1, 1, \dots, 1)$. Setting $w_1 = 0$ evaluates the counterfactual in which the learner observed the interaction at $t=1$ but did not update its policy from it.

\subsection{Replay and conditional intervention}

When asking how the final model would have changed under a perturbation at round $t$, one must decide what happens to the data collected afterward. There are two natural answers.

The first is replay: hold the future data fixed. For a fully realized history $z_{1:T}$, the finite replay effect and the replay influence are
$$\Delta_{t,\epsilon}^{\mathrm{rep}}(F_\nu; z_{1:T})
:=
F_\nu\!\left(\theta_{T+1}^{(w^{(t,\epsilon)})}(z_{1:T})\right)
-
F_\nu\!\left(\theta_{T+1}(z_{1:T})\right),
\qquad
\mathcal I_t^{\mathrm{rep}}(F_\nu; z_{1:T})
:=
\left.
\frac{d}{d\epsilon}
F_\nu\!\left(\theta_{T+1}^{(w^{(t,\epsilon)})}(z_{1:T})\right)
\right|_{\epsilon=0}.$$
Replay is the natural object when the training set is fixed---it asks how the terminal model changes as the perturbed learner is pushed through the same sequence of data.

Before defining the second target, we note a basic causal fact: perturbing the learning weight at round $t$ does not change the past, because the learner only updates its state after observing $Z_t$. Therefore
$$\mathbb P_\nu^{(w^{(t,\epsilon)})}(z_{1:t}) = \mathbb P_\nu(z_{1:t}).$$
This prefix invariance means that conditioning on a realized prefix $z_{1:t}$ of positive baseline probability remains a valid event under the perturbation.

The second target, conditional intervention, holds reality fixed only through round $t$ and then lets the perturbed learner recollect the future. Fix a realized prefix $z_{1:t} \in \mathcal H_t$ with $\mathbb P_\nu(z_{1:t}) > 0$. The perturbed conditional future law is
$$Q_{\nu,t}^{\epsilon}(z_{t+1:T}\mid z_{1:t})
:=
\mathbb P_\nu^{(w^{(t,\epsilon)})}(z_{t+1:T}\mid z_{1:t}),
\qquad
\epsilon\in[-1,\rho].$$
The expected terminal target under this law is
$$\Psi_t^\epsilon(z_{1:t})
:=
\sum_{z_{t+1:T}}
Q_{\nu,t}^{\epsilon}(z_{t+1:T}\mid z_{1:t})\,
F_\nu\!\left(\theta_{T+1}^{(w^{(t,\epsilon)})}(z_{1:T})\right),$$
and the finite interventional effect and conditional interventional influence are
$$\Delta_{t,\epsilon}^{\mathrm{int}}(F_\nu; z_{1:t})
:=
\Psi_t^\epsilon(z_{1:t})-\Psi_t^0(z_{1:t}),
\qquad
\mathcal I_t^{\mathrm{int}}(F_\nu; z_{1:t})
:=
\left.
\frac{d}{d\epsilon}\Psi_t^\epsilon(z_{1:t})
\right|_{\epsilon=0}.$$
This is the primary attribution target of the paper. It attributes the impact of a specific realized occurrence while allowing the post-occurrence data stream to adapt to the altered learner.

Averaging over all realized prefixes at round $t$ gives a slot-level influence:
$$\overline{\mathcal I}_t^{\mathrm{int}}(F_\nu)
:=
\left.
\frac{d}{d\epsilon}
\sum_{z_{1:T}\in\mathcal H_T}
\mathbb P_\nu^{(w^{(t,\epsilon)})}(z_{1:T})
F_\nu\!\left(\theta_{T+1}^{(w^{(t,\epsilon)})}(z_{1:T})\right)
\right|_{\epsilon=0}.$$

The choice to condition on $z_{1:t}$ specifically is deliberate. Conditioning on the full history $z_{1:T}$ freezes all future randomness, collapsing the interventional target into replay. Conditioning on a strict prefix $z_{1:k}$ with $k < t$ loses the identity of the perturbed occurrence, averaging over all possible histories through round $t$. Conditioning on $z_{1:t}$ isolates the exact occurrence being attributed while leaving the causal future free to respond.

\subsection{When recollection collapses to replay}

Fix a round \(t\) and a realized prefix \(h=z_{1:t}\). Suppose that after round
\(t\), the interaction law no longer depends on the learner state. Then a
perturbation at round \(t\) can still change the terminal learner along any
fixed continuation, but it cannot change which continuation is collected.
Recollection and replay therefore coincide at the level of the future data law.
The only remaining choice is how much of the realized history one conditions on.

\begin{proposition}[Exogenous reduction]
\label{prop:exogenous-reduction}
Fix \(t\in\{1,\dots,T\}\) and a realized prefix \(h=z_{1:t}\in\mathcal H_t\)
with \(\mathbb P_\nu(h)>0\). Assume that for every future round
\(s\in\{t+1,\dots,T\}\), every history prefix
\(z_{1:s-1}\in\mathcal H_{s-1}\), and every pair of learner states
\(\theta,\vartheta\in\Theta\),
\[
K_{\nu,s}^{\theta}(\cdot\mid z_{1:s-1})
=
K_{\nu,s}^{\vartheta}(\cdot\mid z_{1:s-1}).
\]
Then for every \(\epsilon\in[-1,\rho]\) and every continuation
\(c\in\mathcal C_t\),
\[
Q_{\nu,t}^{\epsilon}(c\mid h)
=
Q_{\nu,t}^{0}(c\mid h)
=
\mathbb P_\nu(c\mid h).
\]
Consequently,
\[
\Psi_t^\epsilon(h)
=
\mathbb E_{\mathbb P_\nu}
\!\left[
F_\nu\!\left(
\theta_{T+1}^{(w^{(t,\epsilon)})}(Z_{1:T})
\right)
\mid Z_{1:t}=h
\right].
\]
Hence the finite conditional interventional effect reduces to conditional
expected replay:
\[
\Delta_{t,\epsilon}^{\mathrm{int}}(F_\nu;h)
=
\mathbb E_{\mathbb P_\nu}
\!\left[
\Delta_{t,\epsilon}^{\mathrm{rep}}(F_\nu;Z_{1:T})
\mid Z_{1:t}=h
\right].
\]
If the derivative at \(0\) exists, then
\[
\mathcal I_t^{\mathrm{int}}(F_\nu;h)
=
\mathbb E_{\mathbb P_\nu}
\!\left[
\mathcal I_t^{\mathrm{rep}}(F_\nu;Z_{1:T})
\mid Z_{1:t}=h
\right].
\]
\end{proposition}

Proposition~\ref{prop:exogenous-reduction} is the formal bridge to
occurrence-specific attribution on exogenous training streams. In that class,
the future-law correction disappears identically. The adaptive target therefore
reduces to conditional expected replay. If one conditions further on the full
realized continuation \(Z_{1:T}=z_{1:T}\), one recovers the fixed-log,
trajectory-specific viewpoint studied in exogenous run-specific attribution.

\section{Deconstructing the Adaptive Gap}
\label{sec:structural-gap}

\subsection{The Structural Decomposition}

The interventional target $\Psi_t^\epsilon(h)$ is a sum over all possible continuations of the probability of that continuation times the terminal value under the perturbed learner. Taking the derivative in $\epsilon$ by the product rule decomposes the interventional influence into two terms. Fix a realized prefix $h = z_{1:t}$ with $\mathbb P_\nu(h)>0$, and let
\begin{equation}\dot Q_{\nu,t}(c\mid h):=\left.\frac{d}{d\epsilon}Q_{\nu,t}^{\epsilon}(c\mid h)\right|_{\epsilon=0}\label{eq:dotQ}\end{equation}
denote the first-order shift in the conditional future law at the perturbation.

\begin{theorem}[Structural decomposition]
\label{thm:structural-decomposition}
Fix $t\in\{1,\dots,T\}$ and a realized prefix $h=z_{1:t}$ with positive baseline probability. Assuming the maps $\epsilon \mapsto Q_{\nu,t}^{\epsilon}$ and $\epsilon \mapsto F_\nu$ are differentiable at $0$,
\begin{equation}
\mathcal I_t^{\mathrm{int}}(F_\nu; h)
=
\mathbb E_{\mathbb P_\nu}
\!\left[
\mathcal I_t^{\mathrm{rep}}(F_\nu; Z_{1:T})
\mid Z_{1:t}=h
\right]
+
\sum_{c\in\mathcal C_t}
\dot Q_{\nu,t}(c\mid h)
F_\nu\!\left(\theta_{T+1}(h,c)\right).
\label{eq:structural-decomposition}
\end{equation}
\end{theorem}

The first term is conditional expected replay: it accounts for how the perturbed learner changes the terminal model when pushed through the realized future data. In a fixed-dataset setting the future data law does not depend on the learner, so $\dot Q = 0$ and this term is all there is. In an adaptive setting the future data law does depend on the learner, and the second term---the future-law correction---captures the change in which continuations the perturbed learner is likely to encounter. It is this term that replay ignores.

\subsection{The Centered Form}
Because $\dot Q_{\nu,t}(\cdot \mid h)$ sums to zero over all continuations (the total probability mass is always 1), one can subtract a baseline constant from each summand in the correction term without changing the total. This gives a more interpretable form.

\begin{proposition}[Centered form of the future-law term]
\label{prop:centered-future-law}
Under the assumptions of Theorem~\ref{thm:structural-decomposition},
\[
\sum_{c\in\mathcal C_t}
\dot Q_{\nu,t}(c\mid h)=0,
\]
and consequently
\begin{align}
\mathcal I_t^{\mathrm{int}}(F_\nu; h)
-
\mathbb E_{\mathbb P_\nu}
\!\left[
\mathcal I_t^{\mathrm{rep}}(F_\nu; Z_{1:T})
\mid Z_{1:t}=h
\right]
&=
\sum_{c\in\mathcal C_t}
\dot Q_{\nu,t}(c\mid h)
\Bigl[
F_\nu\!\left(\theta_{T+1}(h,c)\right)
-
F_\nu\!\left(\theta_{t+1}(h)\right)
\Bigr].
\label{eq:centered-future-law}
\end{align}
\end{proposition}

Equation~\eqref{eq:centered-future-law} shows that the gap behaves like a covariance between the shift in the future data law and the terminal value of the shifted continuations. The gap is zero when the perturbation does not change which continuations the learner encounters ($\dot Q = 0$), and also when all continuations lead to the same terminal value regardless of which ones become more or less likely. The gap is large when the perturbation meaningfully shifts the data distribution toward continuations with substantially different terminal outcomes.

\section{The Insufficiency of Replay-Side Information}\label{sec:replay-insufficiency}

The structural decomposition shows that replay misses the future-law correction. A natural question is whether richer replay-side information could close the gap. To answer this, we consider the most informative replay-side object one could construct.

For any continuation $c \in \mathcal C_t := \mathcal Z_{t+1}\times\cdots\times\mathcal Z_T$, define the fixed-log replay response curve
$$\phi_{\nu,t,c}(\epsilon; z_{1:t}) := F_\nu\!\left(\theta_{T+1}^{(w^{(t,\epsilon)})}(z_{1:t},c)\right),$$
which records how the terminal evaluation changes as the perturbation level $\epsilon$ varies, with the learner forced through the specific continuation $c$. The replay oracle at prefix $z_{1:t}$ is then
$$\mathcal R_{\nu,t}(z_{1:t}) := \Bigl( Q_{\nu,t}^{0}(\cdot\mid z_{1:t}),\ \{\phi_{\nu,t,c}(\cdot; z_{1:t}) : c \in \mathcal C_t\} \Bigr),$$
pairing the baseline probability of every continuation with the full family of fixed-log response curves for every continuation. This is the richest conceivable replay summary.

\begin{theorem}[Replay-oracle insufficiency]\label{thm:replay-oracle-insufficiency}There exists a smooth horizon-$2$ adaptive learning class and a realized prefix $z_1^\star$ of positive probability such that for every pair $\alpha\neq\beta$, the environments share the same replay oracle:$$\mathcal R_{\nu_\alpha,1}(z_1^\star) = \mathcal R_{\nu_\beta,1}(z_1^\star),$$but their interventional targets differ:$$\Psi_{\nu_\alpha,1}^\epsilon(z_1^\star) \neq \Psi_{\nu_\beta,1}^\epsilon(z_1^\star) \qquad \text{for every }\epsilon \text{ with } \alpha\epsilon\neq\beta\epsilon.$$In particular, $\mathcal I_{1,\nu_\alpha}^{\mathrm{int}}(F; z_1^\star) \neq \mathcal I_{1,\nu_\beta}^{\mathrm{int}}(F; z_1^\star)$, and the interventional target is not identified by any functional of the replay oracle over this class.\end{theorem}

The reason is that the replay oracle, however rich, contains no information about $\dot Q$---the derivative of the future data law with respect to the perturbation. Two environments can agree on every fixed-log replay response curve while disagreeing on how the environment's data-generation mechanism responds to a changed learner state. The oracle captures how the perturbed learner behaves on every conceivable historical log; it cannot capture how the perturbation changes which log the learner would actually encounter.

\section{Identification under Action-Only Learner State Dependence}\label{sec:action-only-frontier}

The negative result rules out identification in arbitrary environments. To recover it, one must restrict how the learner and the environment interact. The natural restriction for contextual bandit and RL settings is that the learner affects the future only through its action policy, not through the context or reward distributions. We call this the action-only class.

Formally, assume each interaction decomposes as $Z_s=(X_s,A_s,R_s)$ and that the transition kernel factors as
\begin{equation}
K_{\nu,s}^{\theta}(x,a,r\mid z_{1:s-1})
=
D_s(x\mid z_{1:s-1})\,
\pi_\theta(a\mid x)\,
P_s(r\mid x,a,z_{1:s-1}),
\label{eq:action-only-factorization}
\end{equation}
where $D_s$ and $P_s$ do not depend on the learner state $\theta$. The environment controls contexts and rewards; the learner controls only its action distribution. Under this structure, perturbing the learner at round $t$ changes the future data law only through the policy, and the change can be tracked exactly. For a realized prefix $h = z_{1:t}$, define the pathwise policy ratio
\begin{equation}\Lambda_t^\epsilon(c;h):=\prod_{s=t+1}^{T}\frac{\pi_{\theta_s^{(w^{(t,\epsilon)})}(h,c_{t+1:s-1})}(a_s\mid x_s)}{\pi_{\theta_s(h,c_{t+1:s-1})}(a_s\mid x_s)}.\label{eq:policy-ratio}\end{equation}
Because the context and reward terms cancel in the ratio of perturbed to baseline future laws, $\Lambda_t^\epsilon$ gives an exact change-of-measure representation.

\begin{theorem}[Exact change of measure in the action-only class]
\label{thm:action-only-com}
Assume the factorization \eqref{eq:action-only-factorization}. Fix \(h=z_{1:t}\) with \(\mathbb P_\nu(h)>0\). Then for every \(\epsilon\in[-1,\rho]\) and every continuation \(c\in\mathcal C_t\) with \(Q_{\nu,t}^{0}(c\mid h)>0\), the perturbed future probability is simply the baseline probability scaled by the policy ratio:
\begin{equation}
Q_{\nu,t}^{\epsilon}(c\mid h)
=
\Lambda_t^\epsilon(c;h)\,
Q_{\nu,t}^{0}(c\mid h).
\label{eq:action-only-density-ratio}
\end{equation}

Suppose in addition that there exists \(\epsilon_0>0\) such that for every \(\epsilon\in(-\epsilon_0,\epsilon_0)\) and every continuation \(c\in\mathcal C_t\),
\begin{equation}
Q_{\nu,t}^{\epsilon}(c\mid h)>0
\quad\Longrightarrow\quad
Q_{\nu,t}^{0}(c\mid h)>0.
\label{eq:action-only-overlap}
\end{equation}
Then for every \(\epsilon\in(-\epsilon_0,\epsilon_0)\),
\begin{equation}
\Psi_t^\epsilon(h)
=
\mathbb E_{\mathbb P_\nu}
\!\left[
\Lambda_t^\epsilon(Z_{t+1:T};h)\,
F_\nu\!\left(\theta_{T+1}^{(w^{(t,\epsilon)})}(Z_{1:T})\right)
\mid Z_{1:t}=h
\right].
\label{eq:action-only-psi-rep}
\end{equation}
\end{theorem}

The representation~\eqref{eq:action-only-psi-rep} is an importance-sampling identity: the interventional target can be evaluated on baseline data by reweighting each realized continuation by the ratio of the perturbed policy's likelihood to the baseline policy's likelihood along that continuation.

Identification then follows if the terminal evaluation $F_\nu$ is a known functional of the baseline law.

\begin{corollary}[Identification in the action-only class]
\label{cor:action-only-identification}
Assume the hypotheses of Theorem~\ref{thm:action-only-com}, assume that the learner update maps \(U_1,\dots,U_T\) and the action law \(\pi_\theta\) are known, and assume that there exists a measurable functional \(\mathfrak F\) such that for every \(\vartheta\in\mathcal R_{\nu,t}^{\epsilon_0}(h)\),
\[
F_\nu(\vartheta)=\mathfrak F(\mathbb P_\nu,\vartheta).
\]
Then for every \(\epsilon\in(-\epsilon_0,\epsilon_0)\), the finite target \(\Psi_t^\epsilon(h)\) is identified from the baseline law. If the derivative at \(0\) exists, then \(\mathcal I_t^{\mathrm{int}}(F_\nu;h)\) is identified as well.
\end{corollary}

If the action-only assumption is relaxed and the environment uses the learner's state to generate future contexts or rewards in an unknown way, identification fails.

\begin{proposition}[Unknown context-state dependence destroys identification]
\label{prop:context-dependence-obstruction}
Consider the class of finite-horizon environments with factorization
\[
K_{\nu,s}^{\theta}(x,a,r\mid z_{1:s-1})
=
D_s^\theta(x\mid z_{1:s-1})\,
\pi_\theta(a\mid x)\,
P_s(r\mid x,a,z_{1:s-1}),
\]
where \(D_s^\theta\) is otherwise unrestricted. Even when the learner update maps and the action law are known, the conditional interventional target is not identified from the baseline law over this class.
\end{proposition}

When unknown learner-state dependence enters context or reward generation, the baseline law does not carry enough information to predict how the future data stream will shift under a perturbation. The action-only class is the structural boundary between identifiable and non-identifiable cases when attribution is done from logged data alone.

\section{Quantifying the Gap between Replay and Intervention}\label{sec:quantifying-the-gap}

We examine how large the gap can be and when it collapses, using a two-armed Bernoulli bandit trained via online entropic mirror descent as a concrete model.

\subsection{The Directional Failure of Replay}

\begin{theorem}[Strong separation]\label{thm:strong-separation}
There exist reward configurations in a horizon-$2$ bandit such that the conditional expected replay is strictly negative while the conditional interventional influence is strictly positive.
\end{theorem}

The sign flip arises from the self-correcting nature of adaptive learning. Suppose the learner pulls a sub-optimal arm and receives a zero reward. Replay looks at this event on its fixed log and concludes that downweighting it would have improved the final model---a negative influence. But the zero reward is what drove the learner to shift toward the better arm for the next round. Intervention sees this: the perturbed model, having downweighted that corrective signal, actually explores less efficiently and ends up worse. The two targets therefore land on opposite sides of zero. Replay, which freezes the future, cannot see the self-correction.

\subsection{The Anatomy of the Gap}
The stagewise bounds derived in the appendix (and stated in the controlled-approximation section) show that the gap is driven by three compounding factors: how strongly the perturbation propagates forward through the learner's parameter updates (model propagation), how much the environment's data distribution shifts in response to a changed learner state (environment sensitivity), and how much the terminal value varies across the resulting trajectories (value oscillation). The sign flip occurs when all three are large simultaneously.

\subsection{The Stable Small-Step Regime}
Controlling these three factors forces the gap to collapse.

\begin{theorem}[Replay in a stable regime]\label{thm:bandit-locality}When the learning rate is sufficiently small and the policy is bounded away from zero and one, the gap between conditional interventional influence and conditional expected replay is $\mathcal{O}(\eta^2)$, where $\eta$ is the learning rate.\end{theorem}

A small learning rate limits how far the perturbation can propagate and how much the environment can react. In this regime the sign-flip scenario disappears and replay is an accurate first-order approximation of the interventional target.

\section{Controlled local approximations}
\label{sec:controlled-approx}

The negative results above show that the global interventional target is not, in general, recoverable from replay-side information alone. At the same time, much of the recent empirical literature does not try to reconstruct the full adaptive future. Instead it works with recent buffers, checkpoint-local targets, replay-side windows, or off-policy surrogates. The right foundational response is therefore not to dismiss locality as heuristic, but to place it on the same counterfactual map and ask what exactly is being approximated.

We begin with the most basic controlled approximation: truncate the \emph{recollection depth}. The idea is to let the perturbation change the future data law for only the next \(L\) rounds after the realized occurrence and then freeze the interaction law at its baseline form. Importantly, this truncates only the \emph{future-law channel}. The perturbed learner still continues to update on the later sampled log all the way to time \(T\). Thus the approximation is local in \emph{recollection depth}, not local in learning dynamics.

Fix \(t\in\{1,\dots,T\}\), a realized prefix \(h=z_{1:t}\in\mathcal H_t\) with \(\mathbb P_\nu(h)>0\), and an integer \(L\in\{0,\dots,T-t\}\). For a continuation \(c=z_{t+1:T}\in\mathcal C_t\), define the depth-\(L\) mixed future law by
\begin{equation}
Q_{\nu,t}^{\epsilon,\mathrm{tr},L}(c\mid h)
:=
\prod_{s=t+1}^{t+L}
K_{\nu,s}^{\theta_s^{(w^{(t,\epsilon)})}(h,c_{t+1:s-1})}
(z_s\mid h,c_{t+1:s-1})
\;
\prod_{s=t+L+1}^{T}
K_{\nu,s}^{\theta_s(h,c_{t+1:s-1})}
(z_s\mid h,c_{t+1:s-1}).
\label{eq:depth-L-mixed-law}
\end{equation}
The first product is empty when \(L=0\), and the second product is empty when \(L=T-t\).

The corresponding depth-\(L\) finite target is
\begin{equation}
\Psi_t^{\epsilon,\mathrm{tr},L}(h)
:=
\sum_{c\in\mathcal C_t}
Q_{\nu,t}^{\epsilon,\mathrm{tr},L}(c\mid h)\,
F_\nu\!\left(\theta_{T+1}^{(w^{(t,\epsilon)})}(h,c)\right),
\label{eq:depth-L-target}
\end{equation}
with finite effect
\[
\Delta_{t,\epsilon}^{\mathrm{tr},L}(F_\nu;h)
:=
\Psi_t^{\epsilon,\mathrm{tr},L}(h)-\Psi_t^0(h),
\]
and infinitesimal form
\begin{equation}
\mathcal I_t^{\mathrm{tr},L}(F_\nu;h)
:=
\left.
\frac{d}{d\epsilon}
\Psi_t^{\epsilon,\mathrm{tr},L}(h)
\right|_{\epsilon=0}.
\label{eq:depth-L-influence}
\end{equation}

This family interpolates exactly between replay and full recollection. When \(L=0\), the future law is baseline all the way to time \(T\), so \(\Psi_t^{\epsilon,\mathrm{tr},0}\) is the conditional expected replay target. When \(L=T-t\), the mixed law is the fully perturbed future law, so \(\Psi_t^{\epsilon,\mathrm{tr},T-t}=\Psi_t^\epsilon\).

To state the next theorem compactly, let \(\Gamma_s\) and \(V_s\) denote the forward state-sensitivity and baseline continuation-value objects from the exact model-based recursion, and define the stagewise contribution of round \(s\) to the replay--intervention gap by
\begin{equation}
\Xi_s(g)
:=
\sum_{z_s\in\mathcal Z_s}
\Bigl(
\nabla_\theta K_{\nu,s}^{\theta}(z_s\mid g)\big|_{\theta=\theta_s(g)}^\top
\Gamma_s(g)
\Bigr)
\bigl(
V_{s+1}(g,z_s)-V_s(g)
\bigr).
\label{eq:stagewise-xi-main}
\end{equation}

\begin{theorem}[Depth-\(L\) recollection identity]
\label{thm:depth-L-recollection}
For every \(L\in\{0,\dots,T-t\}\),
\begin{align}
\mathcal I_t^{\mathrm{tr},L}(F_\nu;h)
&=
\mathbb E_{\mathbb P_\nu}
\!\left[
\mathcal I_t^{\mathrm{rep}}(F_\nu;Z_{1:T})
\mid Z_{1:t}=h
\right]
+
\sum_{s=t+1}^{t+L}
\mathbb E_{\mathbb P_\nu}
\!\left[
\Xi_s(Z_{1:s-1})
\mid Z_{1:t}=h
\right],
\label{eq:depth-L-identity}
\\
\mathcal I_t^{\mathrm{int}}(F_\nu;h)
-
\mathcal I_t^{\mathrm{tr},L}(F_\nu;h)
&=
\sum_{s=t+L+1}^{T}
\mathbb E_{\mathbb P_\nu}
\!\left[
\Xi_s(Z_{1:s-1})
\mid Z_{1:t}=h
\right].
\label{eq:depth-L-tail}
\end{align}
In particular, \(L=0\) recovers conditional expected replay, and \(L=T-t\) recovers full conditional interventional influence.
\end{theorem}

The theorem gives a precise meaning to local attribution windows. A small value of \(L\) does not merely say ``look nearby.'' It says: allow the adaptive future to respond to the perturbation for \(L\) further rounds, then freeze the future data law and continue only the replayed learning dynamics. The bias of this approximation is exactly the omitted tail in \eqref{eq:depth-L-tail}.

To convert that identity into an approximation prescription, recall the total-variation sensitivity and value oscillation quantities
\[
L_s^{\mathrm{TV}}(g)
:=
\sup_{\|u\|=1}
\frac12
\sum_{z_s\in\mathcal Z_s}
\left|
\nabla_\theta K_{\nu,s}^{\theta}(z_s\mid g)\big|_{\theta=\theta_s(g)}^\top u
\right|,
\qquad
\operatorname{osc}_s(g)
:=
\max_{z_s}V_{s+1}(g,z_s)-\min_{z_s}V_{s+1}(g,z_s).
\]

\begin{corollary}[Adaptive-horizon truncation]
\label{cor:adaptive-horizon-truncation}
For every \(L\in\{0,\dots,T-t\}\),
\begin{equation}
\left|
\mathcal I_t^{\mathrm{int}}(F_\nu;h)
-
\mathcal I_t^{\mathrm{tr},L}(F_\nu;h)
\right|
\le
\sum_{s=t+L+1}^{T}
\mathbb E_{\mathbb P_\nu}
\!\left[
L_s^{\mathrm{TV}}(Z_{1:s-1})\,
\|\Gamma_s(Z_{1:s-1})\|\,
\operatorname{osc}_s(Z_{1:s-1})
\mid Z_{1:t}=h
\right].
\label{eq:adaptive-horizon-bound}
\end{equation}
Under the deterministic bounds
\[
\|\partial_w U_t\|\le \bar B_t,
\qquad
\|\partial_\theta U_u\|_{\mathrm{op}}\le \bar\rho_u,
\qquad
L_s^{\mathrm{TV}}\le \bar L_s,
\qquad
\operatorname{osc}_s\le \bar\Delta_s,
\]
this simplifies to
\begin{equation}
\left|
\mathcal I_t^{\mathrm{int}}(F_\nu;h)
-
\mathcal I_t^{\mathrm{tr},L}(F_\nu;h)
\right|
\le
\bar B_t
\sum_{s=t+L+1}^{T}
\bar L_s\,\bar\Delta_s\,
\prod_{u=t+1}^{s-1}\bar\rho_u.
\label{eq:adaptive-horizon-uniform-bound}
\end{equation}
Consequently, for any tolerance \(\tau>0\), any choice of \(L\) satisfying
\[
\bar B_t
\sum_{s=t+L+1}^{T}
\bar L_s\,\bar\Delta_s\,
\prod_{u=t+1}^{s-1}\bar\rho_u
\le
\tau
\]
guarantees a depth-\(L\) approximation error at most \(\tau\).
\end{corollary}

Corollary~\ref{cor:adaptive-horizon-truncation} gives a clean prescription for choosing how local an approximation may be. One first upper-bounds the omitted tail and then selects the smallest horizon \(L\) whose tail falls below a desired tolerance. This is the precise sense in which recent-buffer or checkpoint-local attribution can be principled: they are not arbitrary windows, but approximations to a depth-\(L\) recollection target whose error is controlled by propagation, environment sensitivity, and downstream value oscillation.

This also clarifies how several nearby empirical directions fit into the present theory. Snapshot-style recent-buffer methods are naturally interpreted as scalable proxies for checkpoint-local variants of \(\mathcal I_t^{\mathrm{tr},L}\), obtained by replacing \(F_\nu\) with a checkpoint functional and taking small \(L\) \cite{hu2025snapshot}. Replay-buffer and replay-LOO methods instead stay on the frozen-future side after the switch, even when they become non-local in time \cite{hiraoka2024which,hu2026nonlocal}. Off-policy influence methods for RLVR live on a different approximation axis: they replace online recollection with an off-policy surrogate \cite{zhu2025cropi}. Sample-level post-training influence methods such as RFT-Inf change the attribution unit and target functional rather than only the horizon length \cite{tan2025rftinf}.

The key point is not that every practical method is a literal special case of the global interventional target. Once the future-semantics axis is explicit, one can distinguish when a method targets the same object, when it targets a controlled local truncation, and when it solves a different surrogate problem.

\section{Conclusion}

Data attribution in adaptive learning requires distinguishing two counterfactuals that coincide in static learning but diverge once the learner shapes its own future data. The conditional interventional target defined here captures this distinction: it conditions on what has already occurred and allows the perturbed learner to recollect the future. The main results show that replay-side information cannot recover this target in general, even with full knowledge of the baseline future law and all fixed-log responses, while the action-only class marks the boundary where identification from logged data becomes possible. Whether one targets the full recollection or a depth-$L$ truncation, the choice of future semantics determines what attribution question is actually being asked.

\bibliographystyle{plain}
\bibliography{refs}


\appendix

\section{Additional discussion of related work}

This appendix expands on the literature discussion from the main text, focusing on where the boundaries of the present contribution fall relative to nearby work that is related in spirit but answers different questions.

\subsection{Static data attribution and deletion-by-reweighting}

The classical starting point is influence-function based data attribution~\cite{koh2017understanding}. In that framework, one studies the effect of upweighting a training point by an infinitesimal amount and then relates that differential quantity to deletion by finite reweighting. TracIn, SGD-based tracing methods, and semivalue-based approaches such as Data Shapley all fit within the broader project of attributing model behavior to training data in fixed-dataset learning~\cite{pruthi2020estimating,ghorbani2019data}. The central simplification is that the learner changes the model but does not change which training examples would exist later. This is why replay and recollection collapse in the static setting.

The present paper keeps the deletion-by-reweighting convention. A perturbation with \(\epsilon=-1\) suppresses the direct update contribution of the realized occurrence while leaving the realized prefix event itself intact. This is the standard convention in influence-function style attribution, and it is the natural one for occurrence-level attribution in adaptive learning. What changes in the adaptive setting is not the meaning of deletion-by-reweighting. What changes is that after the perturbation the future data law may also change.

\subsection{The estimand question in fixed-data attribution}

A major methodological lesson from recent attribution work is that one should separate the target from the approximation used to estimate it. Bae et al.\ show that practical influence-function calculations can track a quantity different from literal leave-one-out retraining, namely a proximal response functional that remains meaningful even when exact leave-one-out agreement fails~\cite{bae2022if}. Schioppa et al.\ analyze the assumptions behind the standard leave-one-out reading of influence functions and explain why those assumptions break down in nonlinear, non-convex training~\cite{schioppa2023theoretical}. Distributional training data attribution moves in a related direction by treating training randomness itself as part of the attribution object rather than as a nuisance and by asking how datasets affect the distribution of outputs over training runs~\cite{mlodozeniec2025distributional}.

The present paper adopts the same estimand-first discipline, but the source of extra counterfactual variation is different. Here the issue is not only training randomness over a fixed data source. The issue is that under an occurrence-level perturbation, the learner can change the future data stream itself. This makes the future-law term part of the target rather than a nuisance around a fixed-log computation.

\subsection{Trajectory-specific and run-specific attribution under exogenous data}

Another important line of work studies attribution along a realized optimization trajectory. Hara et al.\ introduced SGD-based occurrence-specific influence, where one asks about removing a point from a specific SGD step rather than from the dataset abstractly~\cite{hara2019sgd}. Wang et al.\ later formalized trajectory-specific leave-one-out influence and emphasized that the same example can have different effects when it appears at different points of training~\cite{wang2025capturing}. In-Run Data Shapley studies run-specific attribution for a target model produced by a single training run, rather than averaging over all runs that the learning algorithm might have produced~\cite{wang2024one-run-shapley}.

These papers are close to the present one in two respects. First, they reject the view that data attribution must be permutation-invariant over training examples. Second, they show that a realized occurrence can be the right unit of analysis. The difference is that the data source remains exogenous. Removing or downweighting a realized occurrence changes the optimization trajectory, but it does not change the distribution of future data. In the language of the present paper, these works study refined forms of replay-side attribution. The present paper asks what happens when the future data source is itself endogenous to the learner.

\subsection{Sequential attribution in reinforcement learning and post-training}

A growing literature studies attribution in RL and other sequential learning systems, but the targets vary considerably.

Gottesman et al.\ analyze influential transitions for off-policy evaluation~\cite{gottesman2020interpretable}. Their target is the sensitivity of an evaluation estimator to transitions in an observational dataset. That is a useful object, especially in high-stakes domains, but it is not the same as asking how a realized training occurrence changes the learned policy together with the future data the learner would collect after that perturbation.

Hiraoka et al.\ study influential experiences for replay-based RL agents~\cite{hiraoka2024which}. This work is also closely related in application domain, but experience replay changes the counterfactual structure. The object of interest is typically a transition already stored in the replay buffer, and the emphasis is on how stored experiences influence the agent through replayed optimization. The present paper instead studies an on-policy occurrence-level perturbation and the way it alters future data collection.

Hu et al.\ are the closest nearby work in terms of motivation~\cite{hu2025snapshot}. They begin from the observation that in online RL each experience both updates the policy and shapes future data collection, which is exactly the phenomenon emphasized here. Their framework, however, is deliberately local. It interprets recent training records relative to a nearby checkpoint and a recent buffer, especially in PPO-style training. The present paper studies a different target: a global occurrence-level counterfactual that conditions on the realized prefix through the perturbed round and then recollects the full remaining future. The contrast is not between right and wrong targets. It is between a local target designed for nearby interpretability and a global target designed to answer the occurrence-level counterfactual question itself.

The same distinction matters in newer post-training settings for language models. Several recent works study data influence in reinforcement fine-tuning or RL-based post-training using local, off-policy, or estimator-specific surrogates. Those developments are important, but they do not by themselves settle the question of what the global attribution target should be once future rollouts depend on the perturbed learner. The present paper is aimed at that earlier step.

\subsection{Sequential causal inference}

The paper also draws a clear line to the causal inference literature on sequential interventions. Longitudinal causal inference has long studied interventions that modify later treatment mechanisms while conditioning on evolving histories, together with identification by g-computation, weighting, and doubly robust methods. The literature on longitudinal modified treatment policies is especially relevant because it treats stochastic, history-dependent interventions that change an observed treatment mechanism without forcing treatment to a fixed value~\cite{diaz2021lmtp}.

We do not claim novelty at that level. The present paper does not propose a new foundation for sequential causal inference. What it contributes is an attribution-theoretic use of that perspective. In adaptive learning, the training history is both a learning trace and a data-collection trace. Once one asks for occurrence-level attribution, sequential intervention ideas become part of the right formal language. The paper then asks three questions that are specific to attribution in adaptive learning: which conditioning level identifies a realized occurrence while still leaving the future free to change, which replay-side objects fail to determine the interventional target, and which structural restrictions recover identification.

\subsection{The contribution in context}

The paper makes three technical claims. First, when future training data are endogenous, replay and recollection are genuinely different counterfactuals, and conditional intervention is the natural global occurrence-level object. Second, the interventional target is not identified by replay-side information alone, even when that information includes the baseline future law and the complete family of fixed-log replay responses for every continuation. Third, in the action-only class---where the learner affects future data only through a known action law---the perturbed future law has an exact change-of-measure form and the target is identified under overlap; unknown learner-state dependence in contexts or rewards breaks identification.

The fixed-data attribution literature has shown that attribution can depend on the realized trajectory, the specific training run, and the exact occurrence at which a point appears. The adaptive-learning literature has recognized that online RL creates a tension between attribution and endogenous future data, but has mostly worked with local or estimator-specific targets. The causal literature on sequential interventions supplies the formal language needed to state the global occurrence-level question cleanly. The present paper brings these threads together: it formulates the target, establishes the replay-side insufficiency, and identifies the structural class where logged-data attribution is possible.

\section{Extended positioning of nearby work under the taxonomy}
\label{app:extended-taxonomy}

This appendix expands on Table~\ref{tab:adaptive-attribution-taxonomy}. The goal is to state carefully which nearby methods are direct special cases, which ones are controlled surrogates, and which ones answer different questions.

\paragraph{Trajectory-specific attribution on exogenous training streams.}
The cleanest direct bridge is to trajectory-specific attribution in ordinary training runs with exogenous data. Hara et al.\ study occurrence-specific influence along SGD trajectories, and Wang et al.\ formalize trajectory-specific leave-one-out influence for removing a data point from a specific iteration of training \cite{hara2019sgd,wang2025capturing}. In our language, these works agree with the present paper on the importance of the realized occurrence, but they lie in the exogenous limit where the perturbation does not change the law of future data. The future-law term therefore vanishes, and replay-side attribution becomes exact.

\paragraph{Replay-buffer experience influence.}
Hiraoka et al.\ study influential experiences stored in a replay buffer and estimate their leave-one-out effects on later RL training \cite{hiraoka2024which}. This is close in application domain but not literally the same process as the present paper. If the replay buffer is treated as fixed, then the counterfactual stays on the replay side. To subsume replay-buffer learning exactly, one would need to enlarge the learner state so that it includes both the policy parameters and the buffer contents, together with two event types: environment-interaction events and replay-update events.

\paragraph{Local and truncated online RL attribution.}
Snapshot studies online RL directly but adopts a deliberately local target, interpreting checkpoints with respect to records in the recent training buffer \cite{hu2025snapshot}. I-PPO similarly uses attribution or filtering at the level of rollout-buffer episodes in PPO-style post-training \cite{liu2026ippo}. These works are best understood as moving along the future-semantics axis from full recollection toward truncated local windows, often together with a checkpoint-local target functional. Section~\ref{sec:controlled-approx} provides the natural formal bridge: a checkpoint-local or short-horizon version of the depth-\(L\) recollection target.

\paragraph{Replay-based non-local on-policy attribution.}
Recent workshop work on non-local attribution for on-policy RL extends the time range of attribution beyond a single recent rollout, but does so using a replay-based leave-one-out objective under fixed rollout buffers \cite{hu2026nonlocal}. This is important to distinguish from the present target. It is non-local in time, but it remains on the replay side of the future-semantics axis. In the language of the present paper, it enlarges the replay window without recollecting the future under the perturbed learner.

\paragraph{Estimator sensitivity and off-policy surrogates.}
Gottesman et al.\ study influential transitions for off-policy evaluation and compute exact influence functions for fitted Q-evaluation and importance-sampling variants \cite{gottesman2020interpretable}. This changes the target functional entirely: the object is the OPE estimator rather than the final learned system after a perturbed occurrence.

CROPI is different again \cite{zhu2025cropi}. It studies RLVR and uses offline trajectories to estimate data influence without fresh online rollouts. In the present taxonomy this is best classified as an \emph{off-policy surrogate} rather than as the same counterfactual object. Conceptually it sits near the positive frontier identified by the action-only class, but theorem-level justification of such off-policy estimators belongs to a second-stage estimation paper rather than to the present foundations paper.

\paragraph{RLVR and RFT sample influence.}
RFT-Inf is especially relevant for post-training \cite{tan2025rftinf}. It defines influence at the level of a training sample, and measures how removing that sample changes the final training reward. This is a genuinely global influence target, but it changes the attribution unit from realized occurrence to dataset sample identity or sample presentation. The exact bridge to the present framework is to decide whether repeated presentations of the same sample are treated as distinct occurrences or collapsed into a single sample-level object. The present paper takes the former route because adaptive learning makes the exact occurrence time and learner state part of the causal object.

\paragraph{Offline world-model valuation and semivalue-style methods.}
Action Shapley studies data valuation for training a world model in RL \cite{zhang2026actionshapley}. This is not occurrence-level adaptive attribution of an online learner changing its own future data. It is an offline data-valuation problem for an exogenous training set. This matters for how one states any semivalue obstruction in the present setting: the obstruction is about realized adaptive occurrences, not about all RL-flavored Shapley methods in general.

\paragraph{Conceptual ancestors and methodological support.}
Several papers are important not because they are literal special cases, but because they support the estimand-first stance of the present work. Bae et al.\ and Schioppa et al.\ argue, in different ways, that one must first decide which counterfactual object an influence method is approximating \cite{bae2022if,schioppa2023theoretical}. Distributional training-data attribution asks how datasets change the distribution of outcomes across training runs, rather than only one endpoint \cite{mlodozeniec2025distributional}. Bottou et al.\ articulate the broader counterfactual-learning-systems perspective in which interventions can change later observations \cite{bottou2013counterfactual}. Longitudinal modified treatment policy work supplies the causal language for stochastic, history-dependent interventions on evolving processes \cite{diaz2021lmtp}. These works are philosophical and methodological supports, not theorem-level reductions of the present paper.

\paragraph{What the taxonomy contributes.}
The value of the taxonomy is not that it forces every nearby method into one theorem. Its value is that it prevents category mistakes. Once the attribution unit, the future semantics, and the target functional are written down separately, it becomes possible to ask whether a nearby method is a direct special case of the global occurrence-level target, whether it is better understood as a controlled truncation or replay-side surrogate, and whether it changes the attribution unit or target functional and therefore answers a different question. The present paper is strongest on the full-recollection corner and on the frontier between identifiable, non-identifiable, replay-side, and truncated targets.

\section{Proofs for Section~\ref{sec:structural-gap}}

\subsection{Proof of Theorem~\ref{thm:structural-decomposition}}

\begin{proof}
Fix \(t\in\{1,\dots,T\}\), fix a realized prefix \(h=z_{1:t}\in\mathcal H_t\) with \(\mathbb P_\nu(h)>0\), and write \(\mathcal C_t=\mathcal Z_{t+1}\times\cdots\times\mathcal Z_T\).

For each continuation \(c\in\mathcal C_t\), define
\[
Q_\epsilon(c):=Q_{\nu,t}^{\epsilon}(c\mid h),
\qquad
G_\epsilon(c):=
F_\nu\!\left(\theta_{T+1}^{(w^{(t,\epsilon)})}(h,c)\right).
\]
Then by the definition of the expected terminal target,
$$
\Psi_t^\epsilon(h)=\sum_{c\in\mathcal C_t} Q_\epsilon(c) G_\epsilon(c).
$$

Therefore, by the definition of conditional interventional influence,
$$
\mathcal I_t^{\mathrm{int}}(F_\nu;h)
=
\left.
\frac{d}{d\epsilon}
\sum_{c\in\mathcal C_t} Q_\epsilon(c)G_\epsilon(c)
\right|_{\epsilon=0}.
$$

The continuation space \(\mathcal C_t\) is finite, so we may differentiate term by term:
\[
\mathcal I_t^{\mathrm{int}}(F_\nu;h)
=
\sum_{c\in\mathcal C_t}
\left.
\frac{d}{d\epsilon}
\bigl[Q_\epsilon(c)G_\epsilon(c)\bigr]
\right|_{\epsilon=0}.
\]
Applying the ordinary product rule to each summand,
\[
\mathcal I_t^{\mathrm{int}}(F_\nu;h)
=
\sum_{c\in\mathcal C_t}
\left.
\frac{d}{d\epsilon}Q_\epsilon(c)
\right|_{\epsilon=0}
G_0(c)
+
\sum_{c\in\mathcal C_t}
Q_0(c)
\left.
\frac{d}{d\epsilon}G_\epsilon(c)
\right|_{\epsilon=0}.
\]

We now identify the two sums.

For the first sum, by definition of \(\dot Q_{\nu,t}(c\mid h)\) in \eqref{eq:dotQ},
\[
\left.
\frac{d}{d\epsilon}Q_\epsilon(c)
\right|_{\epsilon=0}
=
\dot Q_{\nu,t}(c\mid h).
\]
Also,
\[
G_0(c)=F_\nu\!\left(\theta_{T+1}(h,c)\right).
\]
Hence the first sum is
\[
\sum_{c\in\mathcal C_t}
\dot Q_{\nu,t}(c\mid h)
F_\nu\!\left(\theta_{T+1}(h,c)\right).
\]

For the second sum, note that
\[
Q_0(c)=Q_{\nu,t}^{0}(c\mid h)=\mathbb P_\nu(c\mid h),
\]
because \(\epsilon=0\) gives the baseline process.

Moreover, by the definition of replay influence,
$$
\left.
\frac{d}{d\epsilon}
G_\epsilon(c)
\right|_{\epsilon=0}
=
\mathcal I_t^{\mathrm{rep}}(F_\nu;h,c).
$$

Therefore the second sum is
\[
\sum_{c\in\mathcal C_t}
\mathbb P_\nu(c\mid h)
\mathcal I_t^{\mathrm{rep}}(F_\nu;h,c).
\]
By the definition of conditional expectation on a finite space, this is exactly
\[
\mathbb E_{\mathbb P_\nu}
\!\left[
\mathcal I_t^{\mathrm{rep}}(F_\nu;Z_{1:T})
\mid Z_{1:t}=h
\right].
\]

Combining the two sums yields
\[
\mathcal I_t^{\mathrm{int}}(F_\nu;h)
=
\mathbb E_{\mathbb P_\nu}
\!\left[
\mathcal I_t^{\mathrm{rep}}(F_\nu;Z_{1:T})
\mid Z_{1:t}=h
\right]
+
\sum_{c\in\mathcal C_t}
\dot Q_{\nu,t}(c\mid h)
F_\nu\!\left(\theta_{T+1}(h,c)\right),
\]
which is the claimed identity.
\end{proof}

\subsection{Proof of Proposition~\ref{prop:centered-future-law}}

\begin{proof}
Fix \(h=z_{1:t}\) with \(\mathbb P_\nu(h)>0\). For every \(\epsilon\in[-1,\rho]\), the conditional future law \(Q_{\nu,t}^{\epsilon}(\cdot\mid h)\) is a probability distribution on \(\mathcal C_t\). Hence
\[
\sum_{c\in\mathcal C_t} Q_{\nu,t}^{\epsilon}(c\mid h)=1.
\]
Differentiating at \(\epsilon=0\) gives
\[
\sum_{c\in\mathcal C_t} \dot Q_{\nu,t}(c\mid h)=0.
\]

Now subtract the conditional expected replay term from both sides of \eqref{eq:structural-decomposition}. This gives
\[
\mathcal I_t^{\mathrm{int}}(F_\nu;h)
-
\mathbb E_{\mathbb P_\nu}
\!\left[
\mathcal I_t^{\mathrm{rep}}(F_\nu;Z_{1:T})
\mid Z_{1:t}=h
\right]
=
\sum_{c\in\mathcal C_t}
\dot Q_{\nu,t}(c\mid h)
F_\nu\!\left(\theta_{T+1}(h,c)\right).
\]

Because the coefficients \(\dot Q_{\nu,t}(c\mid h)\) sum to zero, we may subtract the same constant from each summand without changing the total. Choose the constant
\[
F_\nu\!\left(\theta_{t+1}(h)\right).
\]
This depends only on the fixed prefix \(h\), not on the continuation \(c\). Therefore
\begin{align*}
\sum_{c\in\mathcal C_t}
\dot Q_{\nu,t}(c\mid h)
F_\nu\!\left(\theta_{T+1}(h,c)\right)
&=
\sum_{c\in\mathcal C_t}
\dot Q_{\nu,t}(c\mid h)
\Bigl[
F_\nu\!\left(\theta_{T+1}(h,c)\right)
-
F_\nu\!\left(\theta_{t+1}(h)\right)
\Bigr].
\end{align*}
Substituting this into the previous display yields \eqref{eq:centered-future-law}.
\end{proof}

\subsection{Proof of Proposition~\ref{prop:exogenous-reduction}}
\label{app:proof-exogenous-reduction}

\begin{proof}
Fix \(t\in\{1,\dots,T\}\) and fix a realized prefix
\(h=z_{1:t}\in\mathcal H_t\) with \(\mathbb P_\nu(h)>0\).

For each future round \(s\in\{t+1,\dots,T\}\) and each history prefix
\(y_{1:s-1}\in\mathcal H_{s-1}\), the hypothesis implies that the kernel does
not depend on the learner state. Therefore there exists a single probability
distribution, which we denote by
\[
\bar K_{\nu,s}(\cdot\mid y_{1:s-1}),
\]
such that
\[
K_{\nu,s}^{\theta}(\cdot\mid y_{1:s-1})
=
\bar K_{\nu,s}(\cdot\mid y_{1:s-1})
\qquad
\text{for every }\theta\in\Theta.
\]

We first identify the perturbed conditional future law. Fix
\(\epsilon\in[-1,\rho]\), and fix a continuation
\[
c=z_{t+1:T}\in\mathcal C_t
=
\mathcal Z_{t+1}\times\cdots\times\mathcal Z_T.
\]
By definition of the conditional future law under the one-coordinate
perturbation \(w^{(t,\epsilon)}\),
\begin{align}
Q_{\nu,t}^{\epsilon}(c\mid h)
&=
\mathbb P_\nu^{(w^{(t,\epsilon)})}(z_{t+1:T}\mid z_{1:t}=h)
\notag\\
&=
\prod_{s=t+1}^{T}
K_{\nu,s}^{\theta_s^{(w^{(t,\epsilon)})}(h,z_{t+1:s-1})}
\!\left(
z_s
\mid
h,z_{t+1:s-1}
\right).
\label{eq:exog-proof-perturbed-future}
\end{align}
Because each future kernel is state-independent, every factor in
\eqref{eq:exog-proof-perturbed-future} equals the common kernel
\(\bar K_{\nu,s}\). Hence
\[
Q_{\nu,t}^{\epsilon}(c\mid h)
=
\prod_{s=t+1}^{T}
\bar K_{\nu,s}(z_s\mid h,z_{t+1:s-1}).
\]
The same calculation holds at \(\epsilon=0\), so
\[
Q_{\nu,t}^{0}(c\mid h)
=
\prod_{s=t+1}^{T}
\bar K_{\nu,s}(z_s\mid h,z_{t+1:s-1}).
\]
Therefore
\[
Q_{\nu,t}^{\epsilon}(c\mid h)
=
Q_{\nu,t}^{0}(c\mid h)
\qquad
\text{for every }c\in\mathcal C_t.
\]
Since \(Q_{\nu,t}^{0}(c\mid h)=\mathbb P_\nu(c\mid h)\) by definition of the
baseline process, we obtain
\[
Q_{\nu,t}^{\epsilon}(c\mid h)
=
Q_{\nu,t}^{0}(c\mid h)
=
\mathbb P_\nu(c\mid h),
\]
which is the first claim.

We now compute the finite interventional target. By definition,
\[
\Psi_t^\epsilon(h)
=
\sum_{c\in\mathcal C_t}
Q_{\nu,t}^{\epsilon}(c\mid h)\,
F_\nu\!\left(
\theta_{T+1}^{(w^{(t,\epsilon)})}(h,c)
\right).
\]
Substituting the already-proved identity for the future law gives
\[
\Psi_t^\epsilon(h)
=
\sum_{c\in\mathcal C_t}
\mathbb P_\nu(c\mid h)\,
F_\nu\!\left(
\theta_{T+1}^{(w^{(t,\epsilon)})}(h,c)
\right).
\]
Because the state space is finite and \(h\) has positive baseline probability,
this is exactly the conditional expectation
\[
\mathbb E_{\mathbb P_\nu}
\!\left[
F_\nu\!\left(
\theta_{T+1}^{(w^{(t,\epsilon)})}(Z_{1:T})
\right)
\mid Z_{1:t}=h
\right].
\]
This proves the displayed formula for \(\Psi_t^\epsilon(h)\).

At \(\epsilon=0\), we similarly have
\[
\Psi_t^0(h)
=
\mathbb E_{\mathbb P_\nu}
\!\left[
F_\nu\!\left(
\theta_{T+1}(Z_{1:T})
\right)
\mid Z_{1:t}=h
\right].
\]
Subtracting the \(\epsilon=0\) identity from the \(\epsilon\)-identity yields
\begin{align*}
\Delta_{t,\epsilon}^{\mathrm{int}}(F_\nu;h)
&=
\Psi_t^\epsilon(h)-\Psi_t^0(h)
\\
&=
\mathbb E_{\mathbb P_\nu}
\!\left[
F_\nu\!\left(
\theta_{T+1}^{(w^{(t,\epsilon)})}(Z_{1:T})
\right)
-
F_\nu\!\left(
\theta_{T+1}(Z_{1:T})
\right)
\mid Z_{1:t}=h
\right]
\\
&=
\mathbb E_{\mathbb P_\nu}
\!\left[
\Delta_{t,\epsilon}^{\mathrm{rep}}(F_\nu;Z_{1:T})
\mid Z_{1:t}=h
\right].
\end{align*}
This proves the finite-effect identity.

Finally, assume the derivative at \(0\) exists. Since \(\mathcal C_t\) is
finite, we may differentiate term by term:
\begin{align*}
\mathcal I_t^{\mathrm{int}}(F_\nu;h)
&=
\left.
\frac{d}{d\epsilon}
\Psi_t^\epsilon(h)
\right|_{\epsilon=0}
\\
&=
\left.
\frac{d}{d\epsilon}
\sum_{c\in\mathcal C_t}
\mathbb P_\nu(c\mid h)\,
F_\nu\!\left(
\theta_{T+1}^{(w^{(t,\epsilon)})}(h,c)
\right)
\right|_{\epsilon=0}
\\
&=
\sum_{c\in\mathcal C_t}
\mathbb P_\nu(c\mid h)\,
\left.
\frac{d}{d\epsilon}
F_\nu\!\left(
\theta_{T+1}^{(w^{(t,\epsilon)})}(h,c)
\right)
\right|_{\epsilon=0}
\\
&=
\sum_{c\in\mathcal C_t}
\mathbb P_\nu(c\mid h)\,
\mathcal I_t^{\mathrm{rep}}(F_\nu;h,c)
\\
&=
\mathbb E_{\mathbb P_\nu}
\!\left[
\mathcal I_t^{\mathrm{rep}}(F_\nu;Z_{1:T})
\mid Z_{1:t}=h
\right].
\end{align*}
This proves the derivative identity.

Under the smoothness assumptions used later in the paper, the same state
independence also implies that \(Q_{\nu,t}^{\epsilon}(\cdot\mid h)\) is
constant in \(\epsilon\), so \(\dot Q_{\nu,t}(\cdot\mid h)\equiv 0\), and that
\(\nabla_\theta K_{\nu,s}^{\theta}(\cdot\mid z_{1:s-1})\equiv 0\) for every
\(s>t\), so every stagewise future-law correction \(\Xi_s\) vanishes as well.
\end{proof}

\section{Proof of Theorem~\ref{thm:replay-oracle-insufficiency}}
\label{app:proof-replay-oracle-insufficiency}

\begin{proof}
We construct the promised smooth class explicitly.

Take horizon \(T=2\). Let
\[
\mathcal Z_1=\{z_1^\star\},
\qquad
\mathcal Z_2=\{0,1\},
\qquad
\Theta=\mathbb R,
\qquad
\theta_1=0.
\]
The first interaction is deterministic:
\[
K_{\nu,1}^{\theta}(z_1^\star\mid \emptyset)=1
\qquad
\text{for every } \nu \text{ and every } \theta.
\]

Define the first update map by
\[
U_1(\theta,z_1^\star,w):=w-1.
\]
Thus under the baseline weight \(w_1=1\),
\[
\theta_2(z_1^\star)=0,
\]
while under the one-coordinate perturbation \(w^{(1,\epsilon)}\),
\[
\theta_2^{(w^{(1,\epsilon)})}(z_1^\star)=\epsilon.
\]

For the second round, define the update and target by
\[
U_2(\theta,z_2,1):=z_2,
\qquad
F(\theta_3):=\theta_3.
\]
Since \(z_2\in\{0,1\}\), this means that the terminal value under any full history \((z_1^\star,z_2)\) is just \(z_2\).

For each parameter \(\gamma\in\mathbb R\), define an environment \(\nu_\gamma\) by the round-\(2\) kernel
\[
K_{\nu_\gamma,2}^{\theta}(1\mid z_1^\star)=\sigma(\gamma\theta),
\qquad
K_{\nu_\gamma,2}^{\theta}(0\mid z_1^\star)=1-\sigma(\gamma\theta),
\]
where \(\sigma(x)=1/(1+e^{-x})\) is the logistic sigmoid.

All these objects are smooth in the obvious sense. The update maps are smooth in their real arguments, the round-\(2\) kernel masses are smooth functions of \(\theta\), and \(F\) is linear.

We now compare the replay oracles at the realized prefix \(z_1^\star\).

First, we compute the baseline future law. Under the baseline process, the first update produces the state \(\theta_2(z_1^\star)=0\). Therefore
\[
Q_{\nu_\gamma,1}^{0}(1\mid z_1^\star)
=
K_{\nu_\gamma,2}^{0}(1\mid z_1^\star)
=
\sigma(0)
=
\frac12,
\]
and similarly
\[
Q_{\nu_\gamma,1}^{0}(0\mid z_1^\star)=\frac12.
\]
Thus the baseline conditional future law at \(z_1^\star\) is the same for every \(\gamma\).

Next, we compute the replay response curves. Fix \(c\in\{0,1\}\). By construction,
\[
\theta_3^{(w^{(1,\epsilon)})}(z_1^\star,c)
=
U_2\!\left(
\theta_2^{(w^{(1,\epsilon)})}(z_1^\star),
c,
1
\right)
=
c.
\]
Therefore
\[
\phi_{\nu_\gamma,1,c}(\epsilon;z_1^\star)
=
F\!\left(\theta_3^{(w^{(1,\epsilon)})}(z_1^\star,c)\right)
=
c
\]
for every \(\epsilon\) and every \(\gamma\). So the entire family of fixed-log replay response curves is also the same for every \(\gamma\).

We have shown that for every \(\alpha,\beta\in\mathbb R\),
\[
\mathcal R_{\nu_\alpha,1}(z_1^\star)
=
\mathcal R_{\nu_\beta,1}(z_1^\star).
\]

We now compute the finite interventional target. Under the perturbation \(w^{(1,\epsilon)}\), the state at time \(2\) is \(\theta_2^\epsilon=\epsilon\). Therefore, conditional on the realized prefix \(z_1^\star\),
\[
Q_{\nu_\gamma,1}^{\epsilon}(1\mid z_1^\star)
=
K_{\nu_\gamma,2}^{\epsilon}(1\mid z_1^\star)
=
\sigma(\gamma\epsilon).
\]
Since the terminal target equals the second-round interaction,
\[
\Psi_{\nu_\gamma,1}^{\epsilon}(z_1^\star)
=
\sum_{c\in\{0,1\}}
Q_{\nu_\gamma,1}^{\epsilon}(c\mid z_1^\star)\,c
=
Q_{\nu_\gamma,1}^{\epsilon}(1\mid z_1^\star)
=
\sigma(\gamma\epsilon).
\]
Hence for \(\alpha\neq\beta\),
\[
\Psi_{\nu_\alpha,1}^{\epsilon}(z_1^\star)
\neq
\Psi_{\nu_\beta,1}^{\epsilon}(z_1^\star)
\]
whenever \(\alpha\epsilon\neq \beta\epsilon\).

Differentiating at \(\epsilon=0\), we obtain
\[
\mathcal I_{1,\nu_\gamma}^{\mathrm{int}}(F;z_1^\star)
=
\left.
\frac{d}{d\epsilon}
\sigma(\gamma\epsilon)
\right|_{\epsilon=0}
=
\gamma \sigma'(0)
=
\frac{\gamma}{4}.
\]
Thus for \(\alpha\neq\beta\),
\[
\mathcal I_{1,\nu_\alpha}^{\mathrm{int}}(F;z_1^\star)
\neq
\mathcal I_{1,\nu_\beta}^{\mathrm{int}}(F;z_1^\star).
\]

Finally, suppose for contradiction that over this class the conditional interventional target were identified by a functional of the replay oracle. Then there would exist a measurable map \(\Phi\) such that for every \(\gamma\),
\[
\mathcal I_{1,\nu_\gamma}^{\mathrm{int}}(F;z_1^\star)
=
\Phi\!\left(\mathcal R_{\nu_\gamma,1}(z_1^\star)\right).
\]
But the replay oracles agree for \(\nu_\alpha\) and \(\nu_\beta\), so this would imply
\[
\mathcal I_{1,\nu_\alpha}^{\mathrm{int}}(F;z_1^\star)
=
\mathcal I_{1,\nu_\beta}^{\mathrm{int}}(F;z_1^\star),
\]
contradicting the calculation above. Therefore the target is not identified by any functional of the replay oracle over this class.
\end{proof}

\section{Proofs for Section~\ref{sec:action-only-frontier}}

\subsection{Proof of Theorem~\ref{thm:action-only-com}}

\begin{proof}
Fix \(t\in\{1,\dots,T\}\), fix a realized prefix \(h=z_{1:t}\in\mathcal H_t\) with \(\mathbb P_\nu(h)>0\), and fix a continuation
\[
c=((x_{t+1},a_{t+1},r_{t+1}),\dots,(x_T,a_T,r_T))\in\mathcal C_t.
\]

We first prove the density-ratio identity \eqref{eq:action-only-density-ratio} for continuations \(c\) with \(Q_{\nu,t}^{0}(c\mid h)>0\).

By the definition of conditional future law and the factorization \eqref{eq:action-only-factorization},
\begin{align}
Q_{\nu,t}^{\epsilon}(c\mid h)
&=
\prod_{s=t+1}^{T}
K_{\nu,s}^{\theta_s^{(w^{(t,\epsilon)})}(h,c_{t+1:s-1})}
(x_s,a_s,r_s\mid h,c_{t+1:s-1})
\notag\\
&=
\prod_{s=t+1}^{T}
D_s(x_s\mid h,c_{t+1:s-1})\,
\pi_{\theta_s^{(w^{(t,\epsilon)})}(h,c_{t+1:s-1})}(a_s\mid x_s)\,
P_s(r_s\mid x_s,a_s,h,c_{t+1:s-1}).
\label{eq:proof-action-only-perturbed}
\end{align}
Similarly, the baseline conditional future law is
\begin{align}
Q_{\nu,t}^{0}(c\mid h)
&=
\prod_{s=t+1}^{T}
D_s(x_s\mid h,c_{t+1:s-1})\,
\pi_{\theta_s(h,c_{t+1:s-1})}(a_s\mid x_s)\,
P_s(r_s\mid x_s,a_s,h,c_{t+1:s-1}).
\label{eq:proof-action-only-baseline}
\end{align}

Because \(Q_{\nu,t}^{0}(c\mid h)>0\), every factor in \eqref{eq:proof-action-only-baseline} corresponding to the actually realized action \(a_s\) is positive, and therefore the ratio
\[
\frac{
\pi_{\theta_s^{(w^{(t,\epsilon)})}(h,c_{t+1:s-1})}(a_s\mid x_s)
}{
\pi_{\theta_s(h,c_{t+1:s-1})}(a_s\mid x_s)
}
\]
is well defined for each \(s\in\{t+1,\dots,T\}\).

Dividing \eqref{eq:proof-action-only-perturbed} by \eqref{eq:proof-action-only-baseline}, the context terms \(D_s\) and reward terms \(P_s\) cancel identically. We obtain
\[
\frac{Q_{\nu,t}^{\epsilon}(c\mid h)}{Q_{\nu,t}^{0}(c\mid h)}
=
\prod_{s=t+1}^{T}
\frac{
\pi_{\theta_s^{(w^{(t,\epsilon)})}(h,c_{t+1:s-1})}(a_s\mid x_s)
}{
\pi_{\theta_s(h,c_{t+1:s-1})}(a_s\mid x_s)
}
=
\Lambda_t^\epsilon(c;h).
\]
Multiplying both sides by \(Q_{\nu,t}^{0}(c\mid h)\) gives
\[
Q_{\nu,t}^{\epsilon}(c\mid h)
=
\Lambda_t^\epsilon(c;h)\,
Q_{\nu,t}^{0}(c\mid h),
\]
which is \eqref{eq:action-only-density-ratio}.

We now prove the expectation representation \eqref{eq:action-only-psi-rep} under the overlap condition \eqref{eq:action-only-overlap}. Fix \(\epsilon\in(-\epsilon_0,\epsilon_0)\). By the definition of \(\Psi_t^\epsilon(h)\),
\[
\Psi_t^\epsilon(h)
=
\sum_{c\in\mathcal C_t}
Q_{\nu,t}^{\epsilon}(c\mid h)
F_\nu\!\left(\theta_{T+1}^{(w^{(t,\epsilon)})}(h,c)\right).
\]
If \(Q_{\nu,t}^{\epsilon}(c\mid h)>0\), the overlap condition implies \(Q_{\nu,t}^{0}(c\mid h)>0\). Therefore the sum may be taken over continuations with positive baseline conditional probability only:
\[
\Psi_t^\epsilon(h)
=
\sum_{c:\,Q_{\nu,t}^{0}(c\mid h)>0}
Q_{\nu,t}^{\epsilon}(c\mid h)
F_\nu\!\left(\theta_{T+1}^{(w^{(t,\epsilon)})}(h,c)\right).
\]
Using the already-proved ratio identity,
\[
\Psi_t^\epsilon(h)
=
\sum_{c:\,Q_{\nu,t}^{0}(c\mid h)>0}
\Lambda_t^\epsilon(c;h)\,
Q_{\nu,t}^{0}(c\mid h)\,
F_\nu\!\left(\theta_{T+1}^{(w^{(t,\epsilon)})}(h,c)\right).
\]
This is exactly the conditional expectation
\[
\mathbb E_{\mathbb P_\nu}
\!\left[
\Lambda_t^\epsilon(Z_{t+1:T};h)\,
F_\nu\!\left(\theta_{T+1}^{(w^{(t,\epsilon)})}(Z_{1:T})\right)
\mid Z_{1:t}=h
\right],
\]
because under the baseline law,
\[
Q_{\nu,t}^{0}(c\mid h)=\mathbb P_\nu(c\mid h).
\]
This proves \eqref{eq:action-only-psi-rep}.
\end{proof}

\subsection{Proof of Corollary~\ref{cor:action-only-identification}}

\begin{proof}
Fix \(\epsilon\in(-\epsilon_0,\epsilon_0)\). By Theorem~\ref{thm:action-only-com},
\[
\Psi_t^\epsilon(h)
=
\mathbb E_{\mathbb P_\nu}
\!\left[
\Lambda_t^\epsilon(Z_{t+1:T};h)\,
F_\nu\!\left(\theta_{T+1}^{(w^{(t,\epsilon)})}(Z_{1:T})\right)
\mid Z_{1:t}=h
\right].
\]

We now show that every quantity inside this conditional expectation is determined by \((\mathbb P_\nu,h)\) under the stated assumptions.

First, because the update maps \(U_1,\dots,U_T\) are known, the realized prefix \(h\) and continuation \(c\) determine the replayed learner states
\[
\theta_s(h,c_{t+1:s-1})
\quad\text{and}\quad
\theta_s^{(w^{(t,\epsilon)})}(h,c_{t+1:s-1})
\]
for every \(s\in\{t+1,\dots,T+1\}\). Therefore the policy ratio
\[
\Lambda_t^\epsilon(c;h)
\]
is determined from \(h\), \(c\), \(\epsilon\), and the known action law \(\pi_\theta\).

Second, for every continuation \(c\) with \(Q_{\nu,t}^{\epsilon}(c\mid h)>0\), the terminal state
\[
\theta_{T+1}^{(w^{(t,\epsilon)})}(h,c)
\]
belongs to the reachable set \(\mathcal R_{\nu,t}^{\epsilon_0}(h)\) by definition. Hence the assumption of the corollary gives
\[
F_\nu\!\left(\theta_{T+1}^{(w^{(t,\epsilon)})}(h,c)\right)
=
\mathfrak F\!\left(
\mathbb P_\nu,
\theta_{T+1}^{(w^{(t,\epsilon)})}(h,c)
\right).
\]
Thus the terminal value appearing inside the expectation is also a measurable functional of \((\mathbb P_\nu,h,c,\epsilon)\).

It follows that the random variable
\[
\Lambda_t^\epsilon(Z_{t+1:T};h)\,
F_\nu\!\left(\theta_{T+1}^{(w^{(t,\epsilon)})}(Z_{1:T})\right)
\]
is measurable with respect to \((\mathbb P_\nu,h,Z_{t+1:T})\). Since the state and continuation spaces are finite, the conditional expectation given \(Z_{1:t}=h\) is itself a measurable functional of \((\mathbb P_\nu,h)\). Therefore \(\Psi_t^\epsilon(h)\) is identified from the baseline law.

If the derivative at \(0\) exists, then
\[
\mathcal I_t^{\mathrm{int}}(F_\nu;h)
=
\left.
\frac{d}{d\epsilon}\Psi_t^\epsilon(h)
\right|_{\epsilon=0}
\]
is also determined by the baseline-law function \(\epsilon\mapsto \Psi_t^\epsilon(h)\), and is therefore identified from the baseline law as well.
\end{proof}

\subsection{Proof of the Negative Frontier}
\begin{proof}
We show directly that the target is not identified from the baseline law over the enlarged class by constructing two environments with the exact same baseline law but strictly divergent interventional targets. We first demonstrate this for reward-state dependence.Take horizon $T=2$. Let the first interaction space be a singleton $\mathcal Z_1=\{z_1^\star\}$, and let the second interaction consist only of a reward, $\mathcal Z_2=\{0,1\}$. Equivalently, one may view the context and action spaces as singletons and the reward space as $\{0,1\}$. Let the learner state space be $\Theta=\mathbb R$, with initial state $\theta_1=0$.Define the first-round kernel to be deterministic:$$K_{\nu,1}^{\theta}(z_1^\star\mid\emptyset)=1$$Define the learner update maps and terminal target by:$$U_1(\theta,z_1^\star,w)=w-1, \qquad U_2(\theta,r,1)=r, \qquad F(\theta_3)=\theta_3$$Therefore, under the perturbation $w^{(1,\epsilon)}$, the learner state after round $1$ is exactly the perturbation value:$$\theta_2^{(w^{(1,\epsilon)})}(z_1^\star)=\epsilon$$For each parameter $\gamma\in\mathbb R$, we define an environment $\nu_\gamma$ by its reward law:$$P_\gamma^\theta(R_2=1\mid z_1^\star)=\sigma(\gamma\theta), \qquad P_\gamma^\theta(R_2=0\mid z_1^\star)=1-\sigma(\gamma\theta)$$where $\sigma$ is the logistic sigmoid function. Because context and action spaces are singletons, this satisfies the factorization for reward-state dependence.We first compute the baseline law. Under the baseline process (where $w_1=1$), the state entering round $2$ is $\theta_2(z_1^\star)=0$. Hence:$$P_\gamma^0(R_2=1\mid z_1^\star)=\sigma(0)=\frac12, \qquad P_\gamma^0(R_2=0\mid z_1^\star)=\frac12$$Since the first round is deterministic, the full baseline law on $\mathcal H_2$ is:$$\mathbb P_{\nu_\gamma}(z_1^\star,1)=\frac12, \qquad \mathbb P_{\nu_\gamma}(z_1^\star,0)=\frac12$$This is entirely independent of $\gamma$. Thus, for any two environments $\alpha,\beta\in\mathbb R$, their baseline laws are identical: $\mathbb P_{\nu_\alpha}=\mathbb P_{\nu_\beta}$.We now compute the interventional target at the realized prefix $z_1^\star$. Under the perturbation $w^{(1,\epsilon)}$, the state at round $2$ is $\epsilon$, so the perturbed future law is:$$Q_{\nu_\gamma,1}^{\epsilon}(1\mid z_1^\star)=\sigma(\gamma\epsilon)$$Because $U_2(\theta,r,1)=r$ and $F(\theta_3)=\theta_3$, the terminal target equals the round-$2$ reward. Therefore, the finite conditional interventional effect is:$$\Psi_{\nu_\gamma,1}^{\epsilon}(z_1^\star) = \sigma(\gamma\epsilon)$$Differentiating at $0$ yields the influence:$$\mathcal I_{1,\nu_\gamma}^{\mathrm{int}}(F;z_1^\star) = \left. \frac{d}{d\epsilon} \sigma(\gamma\epsilon) \right|_{\epsilon=0} = \frac{\gamma}{4}$$Hence, for $\alpha\neq\beta$, their interventional targets strictly diverge:$$\mathcal I_{1,\nu_\alpha}^{\mathrm{int}}(F;z_1^\star) \neq \mathcal I_{1,\nu_\beta}^{\mathrm{int}}(F;z_1^\star)$$If the target were identified from the baseline law over this class, it would have to take the same value on $\nu_\alpha$ and $\nu_\beta$ because their baseline laws agree exactly. Since the values are different, this is impossible.

\paragraph{Extension to Context-State Dependence:}To prove the exact same impossibility for context-dependent environments, we use the identical mathematical construction but treat the second-round interaction as a context $X_2\in\{0,1\}$ rather than a reward. We set $U_2(\theta,x,1)=x$ and define the context law as $D_\gamma^\theta(X_2=1\mid z_1^\star)=\sigma(\gamma\theta)$. The baseline laws again remain perfectly identical across all $\gamma$, but the interventional targets diverge exactly as shown above. Therefore, identification fails in both cases.\end{proof}

\section{Proofs for section~\ref{sec:quantifying-the-gap}}

\subsection{Proofs for the Directional Failure of Replay (Theorem \ref{thm:strong-separation})}

To formally prove the directional failure of replay in the horizon-$2$ bandit (Theorem \ref{thm:strong-separation}), we must first explicitly define the update dynamics of the environment, define the intermediate policy states, and derive the exact closed-form equations for both the interventional and replay targets.

\paragraph{Step 1: The Bandit Update Dynamics.}
The learner updates its policy (the probability $p_t$ of pulling arm $1$) using two-arm entropic mirror-descent with an importance-weighted reward estimate. On the logit scale, where $\text{logit}(p) = \log(p/(1-p))$, the perturbed update rule at round $t$ with learning weight $w_t$ is:
$$
\text{logit}\!\bigl(p_{t+1}^{(w)}\bigr)
=
\text{logit}\!\bigl(p_t^{(w)}\bigr)
+
\eta_t w_t R_t
\left(
\frac{\mathbf 1\{A_t=1\}}{p_t^{(w)}}
-
\frac{\mathbf 1\{A_t=0\}}{1-p_t^{(w)}}
\right)
$$

\paragraph{Step 2: Horizon-2 Setup and the First-Round Derivative.}
We specialize to horizon $T=2$ and condition on the realized first interaction $z_1^\star:=(1,1)$ (the learner pulled arm $1$ and received reward $1$). The initial policy is $p_1=q$. Under the one-coordinate perturbation at round $1$ (where $w_1 = 1+\epsilon$), the intermediate policy entering round $2$ is:
$$
p_2^\epsilon
=
\sigma\!\left(
\text{logit}(q)+\eta_1(1+\epsilon)\frac{1}{q}
\right)
$$
where $\sigma$ is the logistic sigmoid. Differentiating this with respect to $\epsilon$ at $0$ yields the first-round sensitivity factor:
$$
c_{q,\eta_1,p} := \left. \frac{d}{d\epsilon} p_2^\epsilon \right|_{\epsilon=0} = \frac{\eta_1}{q}\,p(1-p)
$$
where $p := p_2^0$ is the baseline intermediate policy.

\paragraph{Step 3: Round-2 Policy Functions.}
We now express the final policy $p_3$ as a function of the intermediate policy $p$. Based on the logit update rule, if the round-$2$ reward is $0$, the policy is unchanged. If the reward is $1$, the update depends on the arm pulled. We define these two potential final policy functions:
$$f_{\eta_2}(p):=\sigma\!\left(\text{logit}(p)+\frac{\eta_2}{p}\right) \quad \text{(if Arm 1 is pulled)}$$
$$g_{\eta_2}(p):=\sigma\!\left(\text{logit}(p)-\frac{\eta_2}{1-p}\right) \quad \text{(if Arm 0 is pulled)}$$

\paragraph{Step 4: Formulating the Targets.}
Using these functions, we define two expectations. The first is the expected final policy after recollecting round $2$ under the true environment parameters $\mu$:
$$
G_{\mu,\eta_2}(p)
:=
p\bigl(\mu_1 f_{\eta_2}(p)+(1-\mu_1)p\bigr)
+
(1-p)\bigl(\mu_0 g_{\eta_2}(p)+(1-\mu_0)p\bigr)
$$
The second is the corresponding conditional expected replay derivative factor (which differentiates the fixed paths rather than the expectations):
$$
R_{\mu,\eta_2}(p)
:=
p\bigl(\mu_1 f_{\eta_2}'(p)+(1-\mu_1)\bigr)
+
(1-p)\bigl(\mu_0 g_{\eta_2}'(p)+(1-\mu_0)\bigr)
$$

\begin{proposition}[Two-step exact formulas]
\label{prop:two-step-formulas}
In the horizon-$2$ bandit conditioned on $Z_1=z_1^\star=(1,1)$, the targets resolve exactly to:
$$ \mathcal I_1^{\mathrm{int}}(F; z_1^\star) = c_{q,\eta_1,p}\,G_{\mu,\eta_2}'(p) $$
$$ \mathbb E_{\mathbb P_\mu} \!\left[ \mathcal I_1^{\mathrm{rep}}(F; Z_{1:2}) \mid Z_1=z_1^\star \right] = c_{q,\eta_1,p}\,R_{\mu,\eta_2}(p) $$
\end{proposition}
\begin{proof}
For the interventional target, we evaluate the recollected expectation given the perturbed intermediate policy $p_2^\epsilon$, so $\Psi_1^\epsilon(z_1^\star)=G_{\mu,\eta_2}(p_2^\epsilon)$. Applying the chain rule yields $G_{\mu,\eta_2}'(p) \cdot \left. \frac{d}{d\epsilon} p_2^\epsilon \right|_{\epsilon=0}$, which equals $c_{q,\eta_1,p}\,G_{\mu,\eta_2}'(p)$.

For the replay target, we evaluate the derivative along each fixed path and then take the baseline expectation. If $(A_2, R_2) = (1,1)$, the final policy is $f_{\eta_2}(p_2^\epsilon)$, and its derivative at $0$ is $c_{q,\eta_1,p}f_{\eta_2}'(p)$. Weighting the derivatives of all four possible round-2 outcomes by their baseline probabilities ($p\mu_1$, $p(1-\mu_1)$, $(1-p)\mu_0$, $(1-p)(1-\mu_0)$) and factoring out $c_{q,\eta_1,p}$ yields exactly $c_{q,\eta_1,p}\,R_{\mu,\eta_2}(p)$.
\end{proof}

\paragraph{Step 5: Auxiliary Half-Point Lemmas.}
To prove the sign separation, we need to evaluate these derivatives at the specific intermediate policy $p=1/2$. 

\begin{lemma}[Half-point identities]
\label{lem:half-point-identities}
For every $\eta>0$, $f_{\eta}'(1/2) = g_{\eta}'(1/2) = -\frac{\eta-1}{\cosh^2(\eta)}$.
\end{lemma}
\begin{proof}
Applying the chain rule to $f_\eta(p)$ and evaluating at $p=1/2$ (where $\text{logit}(1/2)=0$ and $1/(p(1-p))=4$), we find $f_\eta'(1/2) = f_\eta(1/2)(1-f_\eta(1/2)) \cdot 4(1-\eta)$. Using the logistic variance identity $\sigma(x)(1-\sigma(x)) = 1 / (4\cosh^2(x/2))$ evaluated at $x=2\eta$, we obtain $-\frac{\eta-1}{\cosh^2(\eta)}$. The proof for $g_\eta'$ is identical due to symmetry.
\end{proof}

\begin{lemma}[Closed forms at the half point]
\label{lem:R-and-Gprime-half}
Suppose $\mu_1=1$. Then for every $\eta>0$:
$$ R_{\mu,\eta}(1/2) = \frac12 \left[ (1-\mu_0) - (1+\mu_0)\frac{\eta-1}{\cosh^2(\eta)} \right] $$
$$ G_{\mu,\eta}'(1/2) = (1+\mu_0)\sigma(2\eta)-\mu_0 - \frac{1+\mu_0}{2}\frac{\eta-1}{\cosh^2(\eta)} $$
\end{lemma}
\begin{proof}
Substitute $\mu_1=1$ and $p=1/2$ into the definitions of $R_{\mu,\eta}(p)$ and $G_{\mu,\eta}'(p)$. Applying the identities from Lemma \ref{lem:half-point-identities} isolates the stated closed-form equations.
\end{proof}

\paragraph{Step 6: Proof of Theorem \ref{thm:strong-separation} (Strong Separation).}
\begin{proof}
We force the intermediate policy to be exactly $p=1/2$ by setting $q=1/4$ and $\eta_1=(\log 3)/4$. This yields a strictly positive first-round sensitivity: $c_{q,\eta_1,p} = (\log 3)/4 > 0$.

We first analyze replay. By Proposition \ref{prop:two-step-formulas} and Lemma \ref{lem:R-and-Gprime-half}, the replay factor depends on the bracketed term in $R_{\mu,\eta_2}(1/2)$. Because we fix $\eta_2>1$, the term $\frac{\eta_2-1}{\cosh^2(\eta_2)}$ is strictly positive. As $\mu_0 \to 1$, the bracket approaches $0 - 2\frac{\eta_2-1}{\cosh^2(\eta_2)} < 0$. By continuity, there exists a neighborhood $1-\delta_1(\eta_2)<\mu_0<1$ where the conditional expected replay influence is strictly negative.

We now analyze intervention. The target is proportional to $G_{\mu,\eta_2}'(1/2)$. At the limit $\mu_0=1$, this evaluates to $2\sigma(2\eta_2)-1-\frac{\eta_2-1}{\cosh^2(\eta_2)}$, which simplifies to $\tanh(\eta_2)-\frac{\eta_2-1}{\cosh^2(\eta_2)}$. Differentiating this expression with respect to $\eta_2$ reveals its global minimum over $(0,\infty)$ occurs at $\eta_2=1$, where it evaluates to $\tanh(1) > 0$. Therefore, the interventional factor is strictly positive for all $\eta_2>1$ at $\mu_0=1$. By continuity, there exists a neighborhood $1-\delta_2(\eta_2)<\mu_0<1$ where the conditional interventional influence is strictly positive. 

Taking $\delta = \min(\delta_1, \delta_2)$ completes the proof: in this regime, replay evaluates negative while intervention is positive.
\end{proof}

\begin{corollary}[A realized sign flip]
\label{cor:realized-sign-flip}
Under the hypotheses of Theorem \ref{thm:strong-separation}, there exists at least one second-round outcome $z_2^\dagger$ with positive baseline probability such that $\mathcal I_1^{\mathrm{rep}}(F; z_1^\star,z_2^\dagger)<0< \mathcal I_1^{\mathrm{int}}(F; z_1^\star)$.
\end{corollary}
\begin{proof}
By Theorem \ref{thm:strong-separation}, the conditional expected replay is strictly negative. Because an expectation is a convex combination of its realized path values, it is mathematically impossible for every positive-probability continuation to have a non-negative replay influence. Therefore, at least one realized continuation must have a strictly negative replay influence, while the interventional target remains strictly positive.
\end{proof}

\subsection{Proofs for the Anatomy of the Gap (Dynamic Programming and Bounds)}

To rigorously prove the structural bounds discussed in Section \ref{sec:quantifying-the-gap}, we must establish an exact dynamic programming recursion that calculates the interventional target for a known, smooth adaptive model. We assume the learner state space is $\Theta\subseteq \mathbb R^d$, and all relevant update maps $U_s$, kernel masses $K_{\nu,s}$, and the terminal target $F_\nu$ are continuously differentiable.

\paragraph{Step 1: Forward State Sensitivity.}
The perturbation at round $t$ alters the learner's state, and this alteration ripples forward through time. We define the forward sensitivity of the replayed state with respect to the round-$t$ learning weight as $\Gamma$. At the step immediately following the perturbation, this is:
$$ \Gamma_{t+1}(z_{1:t}) := \partial_w U_t(\theta_t(z_{1:t-1}),z_t,1) $$
For all subsequent steps $s=t+1,\dots,T$, the sensitivity propagates via the Jacobian of the update map:
$$ \Gamma_{s+1}(z_{1:s}) := \partial_\theta U_s(\theta_s(z_{1:s-1}),z_s,1)\, \Gamma_s(z_{1:s-1}) $$

\paragraph{Step 2: Backward Continuation Value.}
We next define the baseline expected downstream reward from any given state. At the terminal step $T+1$, this is simply the target function:
$$ V_{T+1}(z_{1:T}) := F_\nu(\theta_{T+1}(z_{1:T})) $$
For prior steps $s=T,\dots,t+1$, we define this recursively by taking the expectation over the baseline next-step interaction:
$$ V_s(z_{1:s-1}) := \sum_{z_s\in\mathcal Z_s} K_{\nu,s}^{\theta_s(z_{1:s-1})}(z_s\mid z_{1:s-1})\, V_{s+1}(z_{1:s}) $$

\paragraph{Step 3: The Exact Combined Target.}
We combine the forward sensitivity and backward value to compute the exact interventional target. Define the terminal gradient sequence:
$$ G_{T+1}(z_{1:T}) := \nabla F_\nu(\theta_{T+1}(z_{1:T}))^\top \Gamma_{T+1}(z_{1:T}) $$
And for $s=T,\dots,t+1$, define the backward recursion:
$$ G_s(z_{1:s-1}) := \sum_{z_s\in\mathcal Z_s} K_{\nu,s}^{\theta_s(z_{1:s-1})}(z_s\mid z_{1:s-1})\, G_{s+1}(z_{1:s}) + \sum_{z_s\in\mathcal Z_s} \Bigl( \nabla_\theta K_{\nu,s}^{\theta}(z_s\mid z_{1:s-1}) \big|_{\theta=\theta_s(z_{1:s-1})}^\top \Gamma_s(z_{1:s-1}) \Bigr)\, V_{s+1}(z_{1:s}) $$

\begin{theorem}[Exact model-based computation]
\label{thm:model-based-exact}
Under the smoothness assumptions above, the interventional target is exactly: $\mathcal I_t^{\mathrm{int}}(F_\nu; h)=G_{t+1}(h)$.
\end{theorem}
\begin{proof}
Let $\theta_s^\epsilon(g) := \theta_s^{(w^{(t,\epsilon)})}(g)$ be the perturbed replayed state. By induction on the replay dynamics, one can verify that its derivative evaluated at $\epsilon=0$ is exactly the forward sensitivity $\Gamma_s(g)$. Next, we define the perturbed continuation value $V_s^\epsilon(g)$ by substituting $\theta^\epsilon$ into the environment kernels and terminal target. By definition, $V_{t+1}^\epsilon(h) = \Psi_t^\epsilon(h)$, meaning the interventional target is exactly $\left. \frac{d}{d\epsilon} V_{t+1}^\epsilon(h) \right|_{\epsilon=0}$.

Let $H_s(g) := \left. \frac{d}{d\epsilon} V_s^\epsilon(g) \right|_{\epsilon=0}$. We show by backward induction that $H_s = G_s$. At $T+1$, the chain rule gives $H_{T+1} = \nabla F_\nu^\top \Gamma_{T+1} = G_{T+1}$. For step $s$, applying the product and chain rules to the recursive definition of $V_s^\epsilon$ yields exactly the two-term sum in the definition of $G_s(z_{1:s-1})$. The first term captures the fixed-law value, and the second captures the shift in the interaction law. Thus, $H_{t+1}(h) = G_{t+1}(h)$.
\end{proof}

\paragraph{Interlude: Future-law score form.}
When the perturbed conditional future law is differentiable and its support is locally stable around \(\epsilon=0\), the future-law correction can also be written with a score. For histories in the support of the baseline conditional future law given \(h\), define
\begin{equation}
S_{\nu,t}(z_{1:T})
:=
\sum_{s=t+1}^{T}
\nabla_\theta
\log K_{\nu,s}^{\theta}(z_s\mid z_{1:s-1})
\big|_{\theta=\theta_s(z_{1:s-1})}^\top
\Gamma_s(z_{1:s-1}).
\label{eq:future-score}
\end{equation}
Setting the score to zero on zero-probability continuations, the replay--intervention gap also admits the centered representation
\begin{equation}
\mathcal I_t^{\mathrm{int}}(F_\nu; h)
-
\mathbb E_{\mathbb P_\nu}
\!\left[
\mathcal I_t^{\mathrm{rep}}(F_\nu; Z_{1:T})
\mid Z_{1:t}=h
\right]
=
\mathbb E_{\mathbb P_\nu}
\!\left[
\bigl(
F_\nu(\theta_{T+1}(Z_{1:T}))
-
F_\nu(\theta_{t+1}(h))
\bigr)
S_{\nu,t}(Z_{1:T})
\mid Z_{1:t}=h
\right].
\label{eq:score-centered}
\end{equation}

\paragraph{Step 4: The Stagewise Decomposition.}
We now isolate the specific contribution of each future round to the total adaptive gap. For each future round $s\in\{t+1,\dots,T\}$, define:
$$ \Xi_s(g) := \sum_{z_s\in\mathcal Z_s} \Bigl( \nabla_\theta K_{\nu,s}^{\theta}(z_s\mid g)\big|_{\theta=\theta_s(g)}^\top \Gamma_s(g) \Bigr) \bigl( V_{s+1}(g,z_s)-V_s(g) \bigr) $$

\begin{theorem}[Stagewise decomposition of the replay--intervention gap]
\label{thm:stagewise-gap}
For every realized prefix $h=z_{1:t}$ with positive baseline probability:
$$ \mathcal I_t^{\mathrm{int}}(F_\nu; h) - \mathbb E_{\mathbb P_\nu} \!\left[ \mathcal I_t^{\mathrm{rep}}(F_\nu; Z_{1:T}) \mid Z_{1:t}=h \right] = \sum_{s=t+1}^{T} \mathbb E_{\mathbb P_\nu} \!\left[ \Xi_s(Z_{1:s-1}) \mid Z_{1:t}=h \right] $$
\end{theorem}
\begin{proof}
Let $M_s(g) := \mathbb E_{\mathbb P_\nu} [ \mathcal I_t^{\mathrm{rep}}(F_\nu;Z_{1:T}) \mid Z_{1:s-1}=g ]$ be the conditional expected replay influence. At $T+1$, $M_{T+1}(z_{1:T}) = G_{T+1}(z_{1:T})$. For prior steps, $M_s(g) = \sum_{z_s} K_{\nu,s}^{\theta_s(g)}(z_s\mid g) M_{s+1}(g,z_s)$. 

Define the gap-to-go $D_s(g) := G_s(g)-M_s(g)$. Subtracting the recursion for $M_s$ from $G_s$ and using the identity $\sum_{z_s} \nabla_\theta K_{\nu,s}^{\theta}(z_s\mid g)^\top \Gamma_s(g) = 0$ to center the value term yields:
$$ D_s(g) = \sum_{z_s\in\mathcal Z_s} K_{\nu,s}^{\theta_s(g)}(z_s\mid g)\, D_{s+1}(g,z_s) + \Xi_s(g) $$
Unrolling this recursion from $s=T+1$ (where $D_{T+1}=0$) down to $t+1$ yields $D_{t+1}(h) = \mathbb E_{\mathbb P_\nu} [ \sum_{s=t+1}^T \Xi_s \mid Z_{1:t}=h ]$, proving the theorem.
\end{proof}

\paragraph{Step 5: Quantitative Bounds.}
Theorem \ref{thm:stagewise-gap} mathematically formalizes the "three gears" discussed in the main text. A future round only contributes to the gap if the forward sensitivity $\Gamma_s$ is non-zero (model propagation), the kernel gradient $\nabla_\theta K$ is non-zero (environment sensitivity), and the downstream values $V_{s+1}$ vary across outcomes (value oscillation). We can bound this by defining local metrics for the environment's total variation sensitivity $L_s^{\mathrm{TV}}(g)$ and the value oscillation $\operatorname{osc}_s(g)$:
$$ L_s^{\mathrm{TV}}(g) := \sup_{\|u\|=1} \frac12 \sum_{z_s\in\mathcal Z_s} \left| \nabla_\theta K_{\nu,s}^{\theta}(z_s\mid g)\big|_{\theta=\theta_s(g)}^\top u \right| \quad \text{and} \quad \operatorname{osc}_s(g) := \max_{z_s} V_{s+1}(g,z_s) - \min_{z_s} V_{s+1}(g,z_s) $$

\begin{corollary}[General oscillation bound]
\label{cor:oscillation-bound}
$$ \left| \mathcal I_t^{\mathrm{int}}(F_\nu; h) - \mathbb E_{\mathbb P_\nu} \!\left[ \mathcal I_t^{\mathrm{rep}}(F_\nu; Z_{1:T}) \mid Z_{1:t}=h \right] \right| \le \sum_{s=t+1}^{T} \mathbb E_{\mathbb P_\nu} \!\left[ L_s^{\mathrm{TV}}(Z_{1:s-1})\, \|\Gamma_s(Z_{1:s-1})\|\, \operatorname{osc}_s(Z_{1:s-1}) \mid Z_{1:t}=h \right] $$
\end{corollary}
\begin{proof}
Centering $V_{s+1}$ around its midpoint $m_g = (\max b_g + \min b_g)/2$ ensures $|V_{s+1} - m_g| \le \frac{1}{2}\operatorname{osc}_s(g)$. Because the kernel gradients sum to zero, shifting by a constant does not change the sum. Factoring out $\operatorname{osc}_s(g)$ and the norm $\|\Gamma_s(g)\|$ leaves the definition of $L_s^{\mathrm{TV}}(g)$. Applying the triangle inequality to Theorem \ref{thm:stagewise-gap} yields the bound.
\end{proof}

\begin{corollary}[Uniform propagation bound]
\label{cor:uniform-propagation}
Assume deterministic bounds: the initial parameter shift $\|\partial_w U_t\| \le \bar B_t$, the update operator norm $\|\partial_\theta U_u\|_{\mathrm{op}} \le \bar\rho_u$, the environment sensitivity $L_s^{\mathrm{TV}} \le \bar L_s$, and the value oscillation $\operatorname{osc}_s \le \bar\Delta_s$. Then the gap is bounded by:
$$ \bar B_t \sum_{s=t+1}^{T} \bar L_s\,\bar\Delta_s \prod_{u=t+1}^{s-1}\bar\rho_u $$
\end{corollary}
\begin{proof}
By the recursive definition of $\Gamma_s$ in Step 1, taking norms gives $\|\Gamma_{s+1}\| \le \bar\rho_s \|\Gamma_s\|$. Iterating this from $t+1$ yields $\|\Gamma_s\| \le \bar B_t \prod_{u=t+1}^{s-1}\bar\rho_u$. Substituting this and the remaining deterministic bounds into Corollary \ref{cor:oscillation-bound} allows us to drop the conditional expectation, yielding the stated uniform bound.
\end{proof}

\subsection{Proofs for the Stable Small-Step Regime (Theorem \ref{thm:bandit-locality})}

To formally prove that a small learning rate mathematically neutralizes the adaptive gap (Theorem \ref{thm:bandit-locality}), we must bound the model propagation and environment sensitivity terms specifically for the two-arm Bernoulli bandit.

\paragraph{Step 1: Replay Sensitivity on a Fixed Log.}
Fix a full realized log $z_{1:T}=((a_1,r_1),\dots,(a_T,r_T))$. Let $\mathcal U_s(p,a,r,w)$ be the one-step policy update map that takes the current probability $p$, the action $a$, the reward $r$, and the learning weight $w$, and outputs the next step's policy via the entropic mirror-descent update defined in Equation \ref{eq:bandit-logit}.

We define the replay sensitivity of the policy at any future time $s > t$ with respect to the round-$t$ perturbation as $D_{s,t}$:
$$ D_{s,t}(z_{1:T}) := \left. \frac{d}{d\epsilon} p_s^{(w^{(t,\epsilon)})}(z_{1:T}) \right|_{\epsilon=0} $$
Because the replayed policy up to time $t$ is independent of the perturbation, $D_{s,t}=0$ for all $s \le t$. At $t+1$, the sensitivity is the direct derivative of the update map: $D_{t+1,t}(z_{1:T}) = \partial_w \mathcal U_t(p_t(z_{1:T}),a_t,r_t,1)$. For all subsequent steps $s > t$, the sensitivity propagates via the chain rule:
$$ D_{s+1,t}(z_{1:T}) = \partial_p \mathcal U_s(p_s(z_{1:T}),a_s,r_s,1)\, D_{s,t}(z_{1:T}) $$

\paragraph{Step 2: The Bandit Future-Law Score.}
We must translate the general environment score $S_{\nu,t}$ into the specific mechanics of the bandit. In this model, the context and reward distributions do not depend on the policy. The only term dependent on the learner's state is the action probability $\pi_p(a_s)$. 

Evaluating the log-derivative $\frac{\partial}{\partial p}\log \pi_p(a_s)$ yields $1/p$ if $a_s=1$ and $-1/(1-p)$ if $a_s=0$. Both cases simplify elegantly to $\frac{a_s-p}{p(1-p)}$. Substituting this and our replay sensitivity $D_{s,t}$ into the general score formula from Equation \ref{eq:future-score} yields the exact bandit score:
$$ S_{\mu,t}(z_{1:T}) = \sum_{s=t+1}^{T} \frac{a_s-p_s(z_{1:T})}{p_s(z_{1:T})(1-p_s(z_{1:T}))}\, D_{s,t}(z_{1:T}) $$

\paragraph{Step 3: Bounding Propagation on the Logit Scale.}
We now impose the stable small-step conditions: the baseline policy is strictly bounded away from the edges by some constant $c \in (0, 1/2)$ such that $c \le p_s \le 1-c$, and the learning rates are bounded such that $\eta_s \le \frac{c}{1-c}$. 

To track propagation cleanly, we evaluate the sensitivity on the logit scale: $X_{s,t} := \left. \frac{d}{d\epsilon} \text{logit}(p_s^\epsilon) \right|_{\epsilon=0}$. Because $p_s = \sigma(x_s)$, the chain rule dictates that $D_{s,t} = p_s(1-p_s)X_{s,t}$.

At the perturbed round, the direct logit update derivative gives $X_{t+1,t} = \eta_t r_t \left( \frac{\mathbf 1\{a_t=1\}}{p_t} - \frac{\mathbf 1\{a_t=0\}}{1-p_t} \right)$. Because $r_t \in \{0,1\}$ and $p_t \ge c$, we can strictly bound the initial shock: $|X_{t+1,t}| \le \eta_t/c$.

For all subsequent rounds $s > t$, if $r_s=0$, the update is inactive and $X_{s+1,t} = X_{s,t}$. If $r_s=1$, differentiating the logit update reveals that the new sensitivity is scaled by a contraction factor. For example, if $a_s=1$, $X_{s+1,t} = \left( 1-\eta_s\frac{1-p_s}{p_s} \right)X_{s,t}$. Under our stable regime bounds, $0 \le \eta_s\frac{1-p_s}{p_s} \le \frac{c}{1-c} \cdot \frac{1-c}{c} = 1$. Because the scaling factor is bounded between $0$ and $1$, the magnitude of the sensitivity never grows: $|X_{s+1,t}| \le |X_{s,t}|$.

By induction, $|X_{s,t}| \le \eta_t/c$ for all future steps. Converting back from the logit scale using the maximum variance $p_s(1-p_s) \le 1/4$, we obtain a uniform bound on the model propagation:
$$ |D_{s,t}| \le \frac{\eta_t}{4c} \qquad \text{for all } s > t $$

\paragraph{Step 4: Proof of Theorem \ref{thm:bandit-locality} (Replay in a stable regime).}
\begin{proof}
We synthesize the bounds to find the maximum gap. First, we bound the future-law score. Since $\left| \frac{a_s-p_s}{p_s(1-p_s)} \right| \le \frac{1}{c}$, substituting our uniform bound for $D_{s,t}$ into the score equation yields:
$$ |S_{\mu,t}(z_{1:T})| \le \sum_{s=t+1}^{T} \frac{1}{c} \cdot \frac{\eta_t}{4c} = \frac{T-t}{4c^2}\eta_t $$

Next, we bound the value oscillation, which depends on the total possible movement of the baseline policy. Differentiating the update map $\mathcal U_u$ with respect to the learning weight reveals a maximum one-step policy shift of $\eta_u/(4c)$. Summing this over all remaining steps bounds the total trajectory divergence:
$$ |p_{T+1}(z_{1:T})-p_{t+1}(h)| \le \frac{1}{4c}\sum_{u=t+1}^{T}\eta_u $$
Assuming the target function $F$ has a bounded derivative $L_F$ over the interval $[c, 1-c]$, the maximum oscillation in the terminal target is bounded by $L_F$ times this trajectory divergence.

Finally, we apply the centered covariance identity (Equation \ref{eq:score-centered}). The gap between the interventional target and replay is the expected product of the target oscillation and the score. Taking the product of our absolute bounds yields the final, deterministic limit:
$$ \left| \mathcal I_t^{\mathrm{int}}(F; h) - \mathbb E_{\mathbb P_\mu} \!\left[ \mathcal I_t^{\mathrm{rep}}(F; Z_{1:T}) \mid Z_{1:t}=h \right] \right| \le \frac{L_F (T-t)}{16c^3}\, \eta_t \sum_{u=t+1}^{T}\eta_u $$
Because this gap is bounded by the product of the initial learning rate $\eta_t$ and the sum of future learning rates $\sum \eta_u$, the total error is strictly $\mathcal O(\eta^2)$.
\end{proof}

\section{Proofs for Section~\ref{sec:controlled-approx}}
\label{app:proofs-controlled-approx}

\subsection{Proof of Theorem~\ref{thm:depth-L-recollection}}

\begin{proof}
Fix \(t\in\{1,\dots,T\}\), fix a realized prefix \(h=z_{1:t}\in\mathcal H_t\) with \(\mathbb P_\nu(h)>0\), and fix \(L\in\{0,\dots,T-t\}\). Write
\[
m:=t+L.
\]
For notational convenience, whenever \(g\) is a prefix of length at least \(u-1\), write
\[
\theta_u^\epsilon(g):=\theta_u^{(w^{(t,\epsilon)})}(g).
\]

\paragraph{Step 1: A backward recursion for the depth-\(L\) target.}
For each \(s\in\{t+1,\dots,T+1\}\) and each prefix \(g=z_{1:s-1}\in\mathcal H_{s-1}\), define recursively
\[
W_{T+1}^{\epsilon,L}(z_{1:T})
:=
F_\nu\!\left(\theta_{T+1}^\epsilon(z_{1:T})\right),
\]
and for \(s=T,T-1,\dots,t+1\),
\begin{equation}
W_s^{\epsilon,L}(g)
:=
\begin{cases}
\displaystyle
\sum_{z_s\in\mathcal Z_s}
K_{\nu,s}^{\theta_s^\epsilon(g)}(z_s\mid g)\,
W_{s+1}^{\epsilon,L}(g,z_s),
& s\le m,\\[3ex]
\displaystyle
\sum_{z_s\in\mathcal Z_s}
K_{\nu,s}^{\theta_s(g)}(z_s\mid g)\,
W_{s+1}^{\epsilon,L}(g,z_s),
& s\ge m+1.
\end{cases}
\label{eq:proof-depth-L-recursion}
\end{equation}
Unrolling the recursion shows that
\[
W_{t+1}^{\epsilon,L}(h)=\Psi_t^{\epsilon,\mathrm{tr},L}(h).
\]
Indeed, for the first \(L\) future rounds the recursion uses the perturbed kernels appearing in \eqref{eq:depth-L-mixed-law}, and for the remaining future rounds it uses the baseline kernels appearing in \eqref{eq:depth-L-mixed-law}; the terminal quantity is exactly the perturbed replay value in \eqref{eq:depth-L-target}.

\paragraph{Step 2: Baseline reduction.}
We claim that for every \(s\in\{t+1,\dots,T+1\}\) and every prefix \(g\in\mathcal H_{s-1}\),
\[
W_s^{0,L}(g)=V_s(g),
\]
where \(V_s\) is the baseline continuation-value recursion introduced earlier.

This is immediate by backward induction on \(s\). At \(s=T+1\),
\[
W_{T+1}^{0,L}(z_{1:T})
=
F_\nu\!\left(\theta_{T+1}(z_{1:T})\right)
=
V_{T+1}(z_{1:T}).
\]
Now suppose \(W_{s+1}^{0,L}=V_{s+1}\). Since \(\theta_s^0(g)=\theta_s(g)\), both branches of \eqref{eq:proof-depth-L-recursion} reduce to
\[
W_s^{0,L}(g)
=
\sum_{z_s\in\mathcal Z_s}
K_{\nu,s}^{\theta_s(g)}(z_s\mid g)\,
V_{s+1}(g,z_s)
=
V_s(g).
\]
So the claim follows.

\paragraph{Step 3: Differentiate the mixed recursion.}
Define
\[
H_s^{L}(g)
:=
\left.
\frac{d}{d\epsilon}
W_s^{\epsilon,L}(g)
\right|_{\epsilon=0}.
\]
Because all interaction spaces are finite and the update maps, kernel masses, and target are continuously differentiable, all derivatives below may be passed through finite sums.

We first record the derivative of the replayed learner state. Exactly as in the forward-sensitivity induction from Theorem~\ref{thm:model-based-exact},
\begin{equation}
\left.
\frac{d}{d\epsilon}
\theta_s^\epsilon(g)
\right|_{\epsilon=0}
=
\Gamma_s(g)
\qquad
\text{for every } s\in\{t+1,\dots,T+1\}.
\label{eq:proof-depth-L-forward-sensitivity}
\end{equation}

At the terminal step,
\[
H_{T+1}^{L}(z_{1:T})
=
\nabla F_\nu\!\left(\theta_{T+1}(z_{1:T})\right)^\top
\Gamma_{T+1}(z_{1:T})
=
G_{T+1}(z_{1:T}).
\]

Now fix \(s\in\{t+1,\dots,T\}\).

If \(s\ge m+1\), then the kernel in \eqref{eq:proof-depth-L-recursion} is frozen at the baseline law and does not depend on \(\epsilon\). Therefore
\begin{equation}
H_s^{L}(g)
=
\sum_{z_s\in\mathcal Z_s}
K_{\nu,s}^{\theta_s(g)}(z_s\mid g)\,
H_{s+1}^{L}(g,z_s).
\label{eq:proof-depth-L-postswitch}
\end{equation}

If \(s\le m\), then the kernel depends on \(\epsilon\) through the perturbed learner state. By the product rule,
\begin{align}
H_s^{L}(g)
&=
\sum_{z_s\in\mathcal Z_s}
\left.
\frac{d}{d\epsilon}
\Bigl[
K_{\nu,s}^{\theta_s^\epsilon(g)}(z_s\mid g)\,
W_{s+1}^{\epsilon,L}(g,z_s)
\Bigr]
\right|_{\epsilon=0}
\notag\\
&=
\sum_{z_s\in\mathcal Z_s}
\left.
\frac{d}{d\epsilon}
K_{\nu,s}^{\theta_s^\epsilon(g)}(z_s\mid g)
\right|_{\epsilon=0}
W_{s+1}^{0,L}(g,z_s)
+
\sum_{z_s\in\mathcal Z_s}
K_{\nu,s}^{\theta_s(g)}(z_s\mid g)\,
H_{s+1}^{L}(g,z_s).
\label{eq:proof-depth-L-prerule}
\end{align}
Using \eqref{eq:proof-depth-L-forward-sensitivity}, the chain rule, and Step 2,
\[
\left.
\frac{d}{d\epsilon}
K_{\nu,s}^{\theta_s^\epsilon(g)}(z_s\mid g)
\right|_{\epsilon=0}
=
\nabla_\theta K_{\nu,s}^{\theta}(z_s\mid g)\big|_{\theta=\theta_s(g)}^\top
\Gamma_s(g),
\qquad
W_{s+1}^{0,L}(g,z_s)=V_{s+1}(g,z_s).
\]
So \eqref{eq:proof-depth-L-prerule} becomes
\begin{align}
H_s^{L}(g)
&=
\sum_{z_s\in\mathcal Z_s}
K_{\nu,s}^{\theta_s(g)}(z_s\mid g)\,
H_{s+1}^{L}(g,z_s)
\notag\\
&\quad+
\sum_{z_s\in\mathcal Z_s}
\Bigl(
\nabla_\theta K_{\nu,s}^{\theta}(z_s\mid g)\big|_{\theta=\theta_s(g)}^\top
\Gamma_s(g)
\Bigr)
V_{s+1}(g,z_s).
\label{eq:proof-depth-L-preswitch}
\end{align}

\paragraph{Step 4: After the switch, the mixed recursion coincides with replay.}
Recall the replay-side backward recursion
\[
M_{T+1}(z_{1:T}) := G_{T+1}(z_{1:T}),
\qquad
M_s(g)
:=
\sum_{z_s\in\mathcal Z_s}
K_{\nu,s}^{\theta_s(g)}(z_s\mid g)\,
M_{s+1}(g,z_s),
\]
introduced in the proof of Theorem~\ref{thm:stagewise-gap}. Whenever \(\mathbb P_\nu(g)>0\), this quantity equals
\[
M_s(g)
=
\mathbb E_{\mathbb P_\nu}
\!\left[
\mathcal I_t^{\mathrm{rep}}(F_\nu;Z_{1:T})
\mid Z_{1:s-1}=g
\right].
\]

Comparing this recursion with \eqref{eq:proof-depth-L-postswitch}, together with the common terminal condition \(H_{T+1}^{L}=M_{T+1}=G_{T+1}\), shows by backward induction that
\begin{equation}
H_s^{L}(g)=M_s(g)
\qquad
\text{for every } s\in\{m+1,\dots,T+1\}.
\label{eq:proof-depth-L-equals-replay-after-switch}
\end{equation}

If \(L=0\), then \(m=t\), so \eqref{eq:proof-depth-L-equals-replay-after-switch} already yields
\[
\mathcal I_t^{\mathrm{tr},0}(F_\nu;h)
=
H_{t+1}^{0}(h)
=
M_{t+1}(h)
=
\mathbb E_{\mathbb P_\nu}
\!\left[
\mathcal I_t^{\mathrm{rep}}(F_\nu;Z_{1:T})
\mid Z_{1:t}=h
\right].
\]
This is exactly \eqref{eq:depth-L-identity} with an empty sum. The tail identity \eqref{eq:depth-L-tail} then follows immediately from Theorem~\ref{thm:stagewise-gap}. It remains to treat the case \(L\ge 1\).

\paragraph{Step 5: Before the switch, the remaining gap is a truncated stagewise sum.}
Assume now that \(L\ge 1\), so \(m\ge t+1\). For \(s\in\{t+1,\dots,m\}\), define
\[
E_s^{L}(g):=H_s^{L}(g)-M_s(g).
\]
By \eqref{eq:proof-depth-L-equals-replay-after-switch},
\[
E_{m+1}^{L}(g)=0
\qquad
\text{for every } g\in\mathcal H_m.
\]

Subtracting the replay recursion for \(M_s\) from \eqref{eq:proof-depth-L-preswitch} yields
\begin{align}
E_s^{L}(g)
&=
\sum_{z_s\in\mathcal Z_s}
K_{\nu,s}^{\theta_s(g)}(z_s\mid g)\,
E_{s+1}^{L}(g,z_s)
\notag\\
&\quad+
\sum_{z_s\in\mathcal Z_s}
\Bigl(
\nabla_\theta K_{\nu,s}^{\theta}(z_s\mid g)\big|_{\theta=\theta_s(g)}^\top
\Gamma_s(g)
\Bigr)
V_{s+1}(g,z_s).
\label{eq:proof-depth-L-gap-raw}
\end{align}

We now center the second sum. Since \(K_{\nu,s}^{\theta}(\cdot\mid g)\) is a probability distribution for every \(\theta\), we have
\[
\sum_{z_s\in\mathcal Z_s}
K_{\nu,s}^{\theta}(z_s\mid g)=1.
\]
Differentiating with respect to \(\theta\) and evaluating at \(\theta=\theta_s(g)\) gives
\[
\sum_{z_s\in\mathcal Z_s}
\nabla_\theta K_{\nu,s}^{\theta}(z_s\mid g)\big|_{\theta=\theta_s(g)}=0.
\]
Therefore
\[
\sum_{z_s\in\mathcal Z_s}
\Bigl(
\nabla_\theta K_{\nu,s}^{\theta}(z_s\mid g)\big|_{\theta=\theta_s(g)}^\top
\Gamma_s(g)
\Bigr)
V_s(g)
=0.
\]
Subtracting this zero term from the second line of \eqref{eq:proof-depth-L-gap-raw} yields
\[
\sum_{z_s\in\mathcal Z_s}
\Bigl(
\nabla_\theta K_{\nu,s}^{\theta}(z_s\mid g)\big|_{\theta=\theta_s(g)}^\top
\Gamma_s(g)
\Bigr)
V_{s+1}(g,z_s)
=
\Xi_s(g),
\]
where \(\Xi_s(g)\) is exactly the stagewise quantity defined in \eqref{eq:stagewise-xi-main}. Hence, for every \(s\in\{t+1,\dots,m\}\),
\begin{equation}
E_s^{L}(g)
=
\sum_{z_s\in\mathcal Z_s}
K_{\nu,s}^{\theta_s(g)}(z_s\mid g)\,
E_{s+1}^{L}(g,z_s)
+
\Xi_s(g),
\qquad
E_{m+1}^{L}(g)=0.
\label{eq:proof-depth-L-gap-centered}
\end{equation}

\paragraph{Step 6: Unroll the truncated gap recursion.}
We claim that whenever \(\mathbb P_\nu(g)>0\),
\begin{equation}
E_s^{L}(g)
=
\sum_{u=s}^{m}
\mathbb E_{\mathbb P_\nu}
\!\left[
\Xi_u(Z_{1:u-1})
\mid Z_{1:s-1}=g
\right]
\qquad
\text{for every } s\in\{t+1,\dots,m\}.
\label{eq:proof-depth-L-unrolled}
\end{equation}

We prove this by backward induction on \(s\).

For \(s=m\), the recursion \eqref{eq:proof-depth-L-gap-centered} and the boundary condition \(E_{m+1}^{L}=0\) give
\[
E_m^{L}(g)=\Xi_m(g),
\]
which is exactly \eqref{eq:proof-depth-L-unrolled}.

Now suppose \eqref{eq:proof-depth-L-unrolled} holds at step \(s+1\). Using \eqref{eq:proof-depth-L-gap-centered},
\begin{align*}
E_s^{L}(g)
&=
\Xi_s(g)
+
\sum_{z_s\in\mathcal Z_s}
K_{\nu,s}^{\theta_s(g)}(z_s\mid g)\,
E_{s+1}^{L}(g,z_s)
\\
&=
\Xi_s(g)
+
\sum_{z_s\in\mathcal Z_s}
K_{\nu,s}^{\theta_s(g)}(z_s\mid g)
\sum_{u=s+1}^{m}
\mathbb E_{\mathbb P_\nu}
\!\left[
\Xi_u(Z_{1:u-1})
\mid Z_{1:s}=(g,z_s)
\right].
\end{align*}
By the tower property of conditional expectation, this equals
\[
\Xi_s(g)
+
\sum_{u=s+1}^{m}
\mathbb E_{\mathbb P_\nu}
\!\left[
\Xi_u(Z_{1:u-1})
\mid Z_{1:s-1}=g
\right].
\]
Since \(\Xi_s(g)\) is measurable with respect to the sigma-field generated by \(Z_{1:s-1}\),
\[
\Xi_s(g)
=
\mathbb E_{\mathbb P_\nu}
\!\left[
\Xi_s(Z_{1:s-1})
\mid Z_{1:s-1}=g
\right].
\]
Thus
\[
E_s^{L}(g)
=
\sum_{u=s}^{m}
\mathbb E_{\mathbb P_\nu}
\!\left[
\Xi_u(Z_{1:u-1})
\mid Z_{1:s-1}=g
\right],
\]
which proves \eqref{eq:proof-depth-L-unrolled}.

Applying \eqref{eq:proof-depth-L-unrolled} at \(s=t+1\) and \(g=h\) gives
\begin{equation}
E_{t+1}^{L}(h)
=
\sum_{u=t+1}^{t+L}
\mathbb E_{\mathbb P_\nu}
\!\left[
\Xi_u(Z_{1:u-1})
\mid Z_{1:t}=h
\right].
\label{eq:proof-depth-L-final-gap}
\end{equation}

\paragraph{Step 7: Identify the truncated influence.}
By construction,
\[
\mathcal I_t^{\mathrm{tr},L}(F_\nu;h)
=
\left.
\frac{d}{d\epsilon}
\Psi_t^{\epsilon,\mathrm{tr},L}(h)
\right|_{\epsilon=0}
=
H_{t+1}^{L}(h).
\]
Using \(H_{t+1}^{L}(h)=M_{t+1}(h)+E_{t+1}^{L}(h)\) together with \eqref{eq:proof-depth-L-final-gap}, we obtain
\[
\mathcal I_t^{\mathrm{tr},L}(F_\nu;h)
=
M_{t+1}(h)
+
\sum_{u=t+1}^{t+L}
\mathbb E_{\mathbb P_\nu}
\!\left[
\Xi_u(Z_{1:u-1})
\mid Z_{1:t}=h
\right].
\]
Because \(\mathbb P_\nu(h)>0\), the replay recursion satisfies
\[
M_{t+1}(h)
=
\mathbb E_{\mathbb P_\nu}
\!\left[
\mathcal I_t^{\mathrm{rep}}(F_\nu;Z_{1:T})
\mid Z_{1:t}=h
\right].
\]
This proves \eqref{eq:depth-L-identity}.

Finally, Theorem~\ref{thm:stagewise-gap} gives
\[
\mathcal I_t^{\mathrm{int}}(F_\nu;h)
=
\mathbb E_{\mathbb P_\nu}
\!\left[
\mathcal I_t^{\mathrm{rep}}(F_\nu;Z_{1:T})
\mid Z_{1:t}=h
\right]
+
\sum_{u=t+1}^{T}
\mathbb E_{\mathbb P_\nu}
\!\left[
\Xi_u(Z_{1:u-1})
\mid Z_{1:t}=h
\right].
\]
Subtracting \eqref{eq:depth-L-identity} yields \eqref{eq:depth-L-tail}. This completes the proof.
\end{proof}

\subsection{Proof of Corollary~\ref{cor:adaptive-horizon-truncation}}

\begin{proof}
By \eqref{eq:depth-L-tail},
\[
\mathcal I_t^{\mathrm{int}}(F_\nu;h)
-
\mathcal I_t^{\mathrm{tr},L}(F_\nu;h)
=
\sum_{s=t+L+1}^{T}
\mathbb E_{\mathbb P_\nu}
\!\left[
\Xi_s(Z_{1:s-1})
\mid Z_{1:t}=h
\right].
\]
Therefore, by the triangle inequality,
\begin{equation}
\left|
\mathcal I_t^{\mathrm{int}}(F_\nu;h)
-
\mathcal I_t^{\mathrm{tr},L}(F_\nu;h)
\right|
\le
\sum_{s=t+L+1}^{T}
\mathbb E_{\mathbb P_\nu}
\!\left[
\left|\Xi_s(Z_{1:s-1})\right|
\mid Z_{1:t}=h
\right].
\label{eq:proof-adaptive-horizon-triangle}
\end{equation}

We now bound \(|\Xi_s(g)|\) pointwise for a fixed prefix \(g=z_{1:s-1}\).

If \(\Gamma_s(g)=0\), then \(\Xi_s(g)=0\), so the desired bound is immediate. Assume therefore that \(\Gamma_s(g)\neq 0\), and set
\[
u_g := \frac{\Gamma_s(g)}{\|\Gamma_s(g)\|}.
\]
Define the midpoint of the downstream values by
\[
m_g
:=
\frac12
\Bigl(
\max_{z_s\in\mathcal Z_s}V_{s+1}(g,z_s)
+
\min_{z_s\in\mathcal Z_s}V_{s+1}(g,z_s)
\Bigr).
\]
Then
\[
\left|V_{s+1}(g,z_s)-m_g\right|
\le
\frac12 \operatorname{osc}_s(g)
\qquad
\text{for every } z_s\in\mathcal Z_s.
\]

As in the proof above, the kernel gradients sum to zero:
\[
\sum_{z_s\in\mathcal Z_s}
\nabla_\theta K_{\nu,s}^{\theta}(z_s\mid g)\big|_{\theta=\theta_s(g)}=0.
\]
Hence
\begin{align*}
\Xi_s(g)
&=
\sum_{z_s\in\mathcal Z_s}
\Bigl(
\nabla_\theta K_{\nu,s}^{\theta}(z_s\mid g)\big|_{\theta=\theta_s(g)}^\top
\Gamma_s(g)
\Bigr)
\Bigl(
V_{s+1}(g,z_s)-m_g
\Bigr).
\end{align*}
Taking absolute values and using the bound above,
\begin{align*}
|\Xi_s(g)|
&\le
\frac12 \operatorname{osc}_s(g)
\sum_{z_s\in\mathcal Z_s}
\left|
\nabla_\theta K_{\nu,s}^{\theta}(z_s\mid g)\big|_{\theta=\theta_s(g)}^\top
\Gamma_s(g)
\right|
\\
&=
\frac12 \operatorname{osc}_s(g)\,\|\Gamma_s(g)\|
\sum_{z_s\in\mathcal Z_s}
\left|
\nabla_\theta K_{\nu,s}^{\theta}(z_s\mid g)\big|_{\theta=\theta_s(g)}^\top
u_g
\right|
\\
&\le
L_s^{\mathrm{TV}}(g)\,
\|\Gamma_s(g)\|\,
\operatorname{osc}_s(g),
\end{align*}
where the last step is exactly the definition of \(L_s^{\mathrm{TV}}(g)\).

Substituting this pointwise bound into \eqref{eq:proof-adaptive-horizon-triangle} yields \eqref{eq:adaptive-horizon-bound}.

For the uniform bound, the forward-sensitivity recursion gives
\[
\|\Gamma_{t+1}(z_{1:t})\|
\le \bar B_t,
\qquad
\|\Gamma_{s+1}(z_{1:s})\|
\le
\bar\rho_s\,\|\Gamma_s(z_{1:s-1})\|
\qquad
\text{for } s=t+1,\dots,T.
\]
Iterating this recursion yields
\[
\|\Gamma_s(z_{1:s-1})\|
\le
\bar B_t
\prod_{u=t+1}^{s-1}\bar\rho_u
\qquad
\text{for every } s\in\{t+1,\dots,T\}.
\]
Combining this with the deterministic bounds \(L_s^{\mathrm{TV}}\le \bar L_s\) and \(\operatorname{osc}_s\le \bar\Delta_s\) inside \eqref{eq:adaptive-horizon-bound} gives
\[
\left|
\mathcal I_t^{\mathrm{int}}(F_\nu;h)
-
\mathcal I_t^{\mathrm{tr},L}(F_\nu;h)
\right|
\le
\bar B_t
\sum_{s=t+L+1}^{T}
\bar L_s\,\bar\Delta_s\,
\prod_{u=t+1}^{s-1}\bar\rho_u,
\]
which is \eqref{eq:adaptive-horizon-uniform-bound}. The final tolerance claim is immediate.
\end{proof}


\newpage

\end{document}